\documentclass[twoside,11pt]{article}
\usepackage{jmlr2e}
\usepackage{amsmath}
\usepackage{amsfonts,amsbsy}
\usepackage{xspace}
\usepackage{graphicx}
\usepackage[font=normalsize]{subfig}

\newtheorem{property}{Property}

\newcommand{\ontop}[2]{\genfrac{}{}{0pt}{}{#1}{#2}}

\DeclareMathOperator{\kl}{kl}
\DeclareMathOperator{\KL}{KL}

\def\inputspace{\mathcal{X}}
\def\outputspace{\mathcal{Y}}
\def\productspace{\mathcal{Z}}
\def\proba{\mathbb{P}}
\def\expectation{\mathbb{E}}
\def\realset{\mathcal{R}}
\def\naturalset{\mathcal{N}}
\def\integerset{\mathcal{Z}}
\def\union{\bigcup}
\def\indicator{\mathbb{I}}
\def\family{\mathcal{H}}
\def\risk{R}

\def\rank{\text{rank}}
\def\pac{\xspace{\sc Pac}}
\def\pefc{\text{\sc Pefc}}
\def\auc{\text{\sc Auc}\xspace}
\def\roc{\text{\sc Roc}\xspace}
\def\sign{\text{sign}}
\def\np{\text{\sc Np}}
\def\normal{\mathcal{N}}
\def\identity{I}
\def\iid{{\sc IID}\xspace}

\def\bfx{{\bf x}}
\def\bfz{{\bf z}}
\def\bfh{{\bf h}}
\def\bfZ{{\bf Z}}
\def\bfz{{\bf z}}
\def\bfD{{\bf D}}
\def\bfC{{\bf C}}
\def\bfP{{\bf P}}
\def\bfQ{{\bf Q}}
\def\bfS{{\bf S}}
\def\bfs{{\bf s}}

\def\bfY{{\bf Y}}

\def\bfy{{\bf y}}
\def\bfpi{\boldsymbol{\pi}}
\def\bfalpha{\boldsymbol{\alpha}}

\usepackage{exscale}
\makeatletter
\newenvironment{equationsize}[1]{%
  \skip@=\baselineskip 
  #1%
  \baselineskip=\skip@ 
  \equation
}{\endequation \ignorespacesafterend} 
\makeatother

\makeatletter
\newenvironment{equationsize*}[1]{%
  \skip@=\baselineskip 
  #1%
  \baselineskip=\skip@ 
  \equation
}{\nonumber\endequation \ignorespacesafterend} 
\makeatother

 

\makeatletter
\newenvironment{alignsize*}[1]{%
  \skip@=\baselineskip 
  #1%
  \baselineskip=\skip@ 
  \start@align\@ne\st@rredtrue\m@ne
}{\endalign\ignorespacesafterend} 
\makeatother

\jmlrheading{}{}{}{}{}{Liva Ralaivola, Marie Szafranski and Guillaume Stempfel}
\ShortHeadings{Chromatic \pac-Bayes Bounds and Non-\iid Data}{Ralaivola,
  Szafranski, Stempfel}
\firstpageno{1}

\title{Chromatic PAC-Bayes Bounds for Non-IID Data:
  Applications to Ranking and Stationary $\beta$-Mixing Processes}

\author{\name Liva Ralaivola \email liva.ralaivola@lif.univ-mrs.fr\\
\name Marie Szafranski \email marie.szafranski@lif.univ-mrs.fr\\
\name Guillaume Stempfel \email guillaume.stempfel@lif.univ-mrs.fr\\
\addr Laboratoire d'Informatique Fondamentale de Marseille\\
\addr CNRS, Aix-Marseille Universit\'es\\
\addr 39, rue F. Joliot Curie, 13013 Marseille, France
}

\begin{document}

\editor{}
\maketitle
\begin{abstract} \pac-Bayes bounds are among the most accurate
  generalization bounds for classifiers learned from independently and
  identically distributed (\iid) data, and it
  is particularly so for margin classifiers: there have been recent
  contributions showing how practical these bounds can be either to
  perform model selection \citep{ambroladze07tighter} or even to directly guide the
  learning of linear classifiers \citep{germain09pacbayesian}. However, there are many
  practical situations where the training data show some dependencies and
  where the traditional \iid assumption does not hold. Stating
  generalization bounds for such frameworks is therefore of the utmost
  interest, both from theoretical and practical standpoints.  In this
  work, we propose the first --~to the best of our knowledge~--
  \pac-Bayes generalization bounds for classifiers trained on data
  exhibiting interdependencies. The approach undertaken to establish our
  results is based on the decomposition of a so-called dependency
  graph that encodes the dependencies within the data, in sets of
  independent data, thanks to graph {\em fractional covers}.  Our
  bounds are very general, since being able to find an upper bound on
  the fractional chromatic number of the dependency graph is
  sufficient to get new \pac-Bayes bounds for specific settings. 
  We show how our results can be used to derive bounds for 
  ranking statistics (such as \auc) and classifiers trained on data
  distributed according to a stationary 
  $\beta$-mixing process. In the way, we show how our approach
  seemlessly allows us to deal with U-processes. As a side note, we
  also provide a \pac-Bayes generalization bound for classifiers
  learned on data from stationary $\varphi$-mixing distributions.
\end{abstract}

\section{Introduction}
\label{sec:introduction}

\subsection{Background}
Recently, there has been much progress in the field of generalization
bounds for classifiers, the most noticeable of which are
Rademacher-complexity-based bounds
\citep{bartlett02rademacher,bartlett05local}, stability-based bounds
\citep{bousquet02stability} and \pac-Bayes bounds
\citep{mcallester99some}. \pac-Bayes bounds, introduced by
\cite{mcallester99some}, and refined in several occasions 
\citep{seeger02pac,langford05tutorial,audibert07combining}, are some of
the most appealing advances from the tightness and accuracy points of
view (an excellent monograph on the \pac-Bayesian framework is that of
\cite{catoni07pacbayesian}).  Among others, striking results have been
obtained concerning \pac-Bayes bounds for linear classifiers:
\cite{ambroladze07tighter} showed that \pac-Bayes bounds are a viable
route to do actual model selection; \cite{germain09pacbayesian}
recently proposed to learn linear classifiers by directly minimizing
the linear \pac-Bayes bound with conclusive results, while
\cite{langford02pac} showed that under some margin assumption,
the \pac-Bayes framework allows one to tightly bound not only the risk
of the stochastic Gibbs classsifier (see below) but also the risk of
the Bayes classifier. The variety of  (algorithmic, theoretical, practical) outcomes
that can be expected from original contributions in the \pac-Bayesian setting
explains and justifies the increasing interest it generates. 

\subsection{Contribution}
To the best of our knowledge, \pac-Bayes bounds have
essentially been derived for the setting where the training data are
{\em independently and identically distributed} (\iid).  Yet, being
able to learn from non-\iid data while having strong theoretical
guarantees on the generalization properties of the learned classifier
is an actual problem in a number of real world applications such as,
e.g., bipartite ranking (and more generally $k$-partite ranking) or classification from sequential data.
Here, we propose the first \pac-Bayes bounds for classifiers trained
on non-\iid data; they constitute a generalization of the \iid
\pac-Bayes bound and they are generic enough to provide a principled
way to establish generalization bounds for a number of non-\iid
settings.  To establish these bounds, we make use of simple tools from
probability theory, convexity properties of some functions, and we
exploit the notion of {\em fractional covers} of graphs
\citep{schreinerman97fractional}.  One way to get a high level view of
our contribution is the following: fractional covers allow us to cope
with the dependencies within the set of random variables at hand by providing
a strategy to make (large) subsets of independent random variables on
which the usual \iid \pac-Bayes bound is applied. Note that
we essentially provide bounds for the case of {\em identically and non-independently}
distributed data; the additional results that we give in the appendix generalizes
to {\em non-identically and non-independently} distributed data.

\subsection{Related Results}
We would like to mention that the idea of dealing with sums of
interdependent random variables by separating them into subsets of
independent variables to establish concentration inequalities dates
back to the work of \cite{hoeffding48class,hoeffding63probability} on
U-statistics. Explicity using the notion of (fractional) covers --~or
equivalently, colorings~-- of graphs to derive such concentration
inequalities has been proposed by \cite{pemmaraju01equitable} and \cite{janson04large} 
and later extended by
\cite{usunier06generalization} to deal with functions that are
different from the sum. Just as \cite{usunier06generalization}, who
used their concentration inequality to provide generalization bounds
based on the {\em fractional Rademacher complexity}, we take the
approach of decomposing a set of dependent random variables into
subsets of dependent random variables a step beyond establishing
concentration inequality to provide what we call {\em chromatic}
\pac-Bayes generalization bounds.

The genericity of our bounds is illustrated in several ways. It allows
us to derive generalization bounds on the ranking performance of
scoring/ranking functions using two different performance measures,
among which the {\em Area under the \roc curve} (\auc) .  These bounds
are directly related to the work of
\cite{agarwal05generalization}, \cite{agarwal09stability},
\cite{clemencon08ranking} and \cite{freund03efficient}. Even if our
bounds are obtained as simple specific instances of our generic \pac-Bayes
bounds, they exhibit interesting peculiarities. Compared with the
bound of \cite{agarwal05generalization} and \cite{freund03efficient},
our \auc bound depends in a less stronger way on the {\em skew}
(i.e. the imbalance between positive and negative data) of the
distribution; besides it does not rest on (rank-)shatter
coefficients/VC dimension that may sometimes be hard to assess
accurately; in addition, our bound directly applies to (kernel-based)
linear classifiers.  \cite{agarwal09stability} base their analysis of
ranking performances on algorithmic stability, and the qualitative
comparison of their bounds and ours is not straightforward because
stability arguments are somewhat different from the arguments used for
\pac-Bayes bounds (and other uniform bounds).  As already observed by
\cite{janson04large}, coloring provides a way to generalize large
deviation results based on U-statistics; this observation carries over
when generalization bounds are considered, which allows us to draw a
connection between the results we obtain and that of
\cite{clemencon08ranking}.

Another illustration of the genericity of our approach deals with
mixing processes. In particular, we show how our chromatic bounds can
be used to easily derive new generalization bounds for $\beta$-mixing
processes. Rademacher complexity based bounds for such type of
processes have recently been established by
\cite{mohri09rademacher}. To the best of our knowledge, it is the
first time that such a bound is given in the \pac-Bayes framework. The
striking feature is that it is done at a very low price: the
independent block method proposed by \cite{yu94rates} directly gives a
dependency graph whose chromatic number is straightforward to
compute. As we shall see, this suffices to instantiate our
chromatic bounds, which, after simple calculations, leads to
appropriate generalization bound.  For sake of completeness, we
also provide a \pac-Bayes bound for stationary $\varphi$-mixing processes;
it is based on a different approach and its presentation is postponed to the
appendix together with the tools that allows us to derive it.

\subsection{Organization of the Paper}
The paper is organized as follows. Section~\ref{sec:iidbound} recalls
the standard \iid \pac-Bayes bound. Section~\ref{sec:bounds} introduces the notion of
fractional covers and states the new chromatic \pac-Bayes
  bounds, which rely on the fractional chromatic number of the {\em
  dependency graph} of the data at hand. Section~\ref{sec:examples}
provides specific versions of our bounds for the case of \iid data, ranking
and stationary $\beta$-mixing processes, giving rise to original generalization bounds.
A \pac-Bayes bound for stationary $\varphi$-mixing based on arguments different from
the chromatic \pac-Bayes bound is provided, in the appendix.


\section{IID \pac-Bayes Bound}
\label{sec:iidbound}

We introduce notation that will hold from here on. We mainly consider the
problem of binary classification over the {\em input space}
$\inputspace$ and we denote the set of possible labels as
$\outputspace=\{-1,+1\}$ (for the case of ranking described in section~\ref{sec:examples}, we
we use $\outputspace=\realset$); $\productspace$ denotes the product space
$\inputspace\times\outputspace$.
$\family\subseteq\realset^{\inputspace}$ is a family of
real valued classifiers defined on $\inputspace$: for $h\in\family$,
the predicted output of $x\in\inputspace$ is given by $\sign(h(x))$, where $\sign(x)=+1$ if $x\geq 0$ and $-1$ otherwise. $D$ is a probability
distribution defined over $\productspace$ and $\bfD_m$ denotes the
distribution of an $m$-sample; for instance,
$\bfD_m=\otimes_{i=1}^mD=D^m$ is the distribution of an \iid sample
$\bfZ=\{Z_i\}_{i=1}^m$ of size $m$ ($Z_i\sim D$, $i=1\ldots m$). $P$
and $Q$ are distributions over $\family$. For any positive integer
$m$, $[m]$ stands for $\{1,\ldots,m\}$.

The \iid \pac-Bayes bound, can be stated as follows
\citep{mcallester03simplified,seeger02pac,langford05tutorial}.
\begin{theorem}[IID \pac-Bayes Bound] \label{th:pbiid}
  $\forall D$, $\forall\family$, $\forall\delta\in(0,1]$, $\forall P$,
  with probability at least $1-\delta$ over the random draw of
  $\bfZ\sim\bfD_m=D^m$, the following holds:
  \begin{equationsize}{\normalsize} \forall Q,\;
    \kl(\hat{e}_Q(\bfZ)||e_Q)\leq\frac{1}{m}\left[\KL(Q||P)+\ln\frac{m+1}{\delta}\right].\label{eq:pbiid}
  \end{equationsize}
\end{theorem}
This theorem provides a
generalization error bound for the {\em Gibbs classifier} $g_Q$: given
a distribution $Q$, this stochastic classifier predicts a class for
$\bfx\in\inputspace$ by first drawing a hypothesis $h$ according to
$Q$ and then outputting $\sign(h(\bfx))$. Here, $\hat{e}_Q$ is the empirical
error of $g_Q$ on an \iid sample $\bfZ$ of size $m$ and $e_Q$ is its
true error: \begin{equationsize}{\normalsize}
  \begin{array}{ll}\displaystyle
    \hat{e}_Q(\bfZ):=\expectation_{h\sim Q}\frac{1}{m}\sum_{i=1}^mr(h,Z_i)=\expectation_{h\sim Q}\hat{\risk}(h,\bfZ)&\text{with } \hat{\risk}(h,\bfZ):=\frac{1}{m}\sum_{i=1}^mr(h,Z_i)\\
    e_Q:=\expectation_{\bfZ\sim\bfD_m}\hat{e}_Q(\bfZ)=\expectation_{h\sim Q}\risk(h) &\text{with } \risk(h):=\expectation_{\bfZ\sim \bfD_m}\hat{\risk}(h,\bfZ),
\end{array}
\label{eq:gibbs}
\end{equationsize}
where, for $Z=(X,Y)$, $$r(h,Z):=\indicator_{Yh(X)<0}.$$ 
Note that we will use this binary 0-1 risk function $r$ throughout the paper and that a generalization of our results to bounded real-valued risk functions is given in appendix. Since $\bfZ$ is an (independently) identically distributed sample, we have
\begin{equation}
\label{eq:iidrisk}
\risk(h)=\expectation_{\bfZ\sim \bfD_m}\hat{\risk}(h,\bfZ)=\expectation_{Z\sim D}r(h,Z).
\end{equation}
 For $p,q\in[0,1]$, $\kl(q||p)$ is the Kullback-Leibler
divergence between the Bernoulli distributions with probabilities of success $q$ and $p$, and $\KL(Q||P)$
is the Kullback-Leibler divergence between $Q$ and $P$:
\begin{align*}
\kl(q||p)&:=q\ln\frac{q}{p}+(1-q)\ln\frac{1-q}{1-p}\\
\KL(Q||P)&:=\expectation_{h\sim Q}\ln\frac{Q(h)}{P(h)},
\end{align*}
where $\kl(0||0)=\kl(1||1)=0$. All along, we assume that the posteriors are absolutely continuous
with respect to their corresponding priors.

It is straightforward to see that the mapping $\kl_q:t\mapsto \kl(q||q+t)$ is strictly increasing for $t\in[0,1-q)$ and therefore defines a bijection from $[0,1-q)$ to $\realset_+$: we denote by  $\kl_q^{-1}$ its inverse. Then, as pointed out by \cite{seeger02pac}, 
the function $\kl^{-1}:(q,\varepsilon)\mapsto\kl^{-1}(q,\varepsilon)=\kl_q^{-1}(\varepsilon)$ is well-defined over $[0,1)\times \realset^+$, and, by definition: $$t\geq\kl^{-1}(q,\varepsilon)\Leftrightarrow \kl(q||q+t)\geq\varepsilon.$$ This makes it possible to rewrite bound \eqref{eq:pbiid} in a more `usual' form:
\begin{equationsize}{\normalsize} \forall Q,\;
 e_Q\leq\hat{e}_Q(\bfZ)+\kl^{-1}\left(\hat{e}_Q(\bfZ),\frac{1}{m}\left[\KL(Q||P)+\ln\frac{m+1}{\delta}\right]\right).\label{eq:pbiidinvkl}
\end{equationsize}

We observe that even if bounds \eqref{eq:pbiid} and \eqref{eq:pbiidinvkl} apply to the risk $e_Q$ of
the stochastic classifier $g_Q$, a straightforward argument gives
that, if $b_Q$ is the (deterministic) Bayes classifier such that
$b_Q(x)=\sign(\expectation_{h\sim Q}h(x))$, then
$R(b_Q)=\expectation_{Z\sim D}r(b_Q,Z)\leq 2e_Q$ (see for instance
\citep{herbrich01pacbayesian}). \cite{langford02pac} show that under
some margin assumption, $R(b_Q)$ can be bounded even more tightly.


\section{Chromatic \pac-Bayes Bounds}
\label{sec:bounds}

The problem we focus on is that of generalizing Theorem~\ref{th:pbiid}
to the situation where there may exist probabilistic dependencies
between the elements $Z_i$ of $\bfZ=\{Z_i\}_{i=1}^m$ while the
marginal distributions of the $Z_i$'s are identical. As announced
before, we provide \pac-Bayes bounds for classifiers trained on
identically but not independently distributed data. These results rely
on properties of a dependency graph that is built according to the
dependencies within $\bfZ$. Before stating our new bounds, we thus
introduce the concepts of graph theory that will play a role in their
statements.

\subsection{Dependency Graph, Fractional Covers }
\begin{definition}[Dependency Graph]
Let ${\bf Z}=\{Z_i\}_{i=1}^m$ be a set of $m$ random variables taking values in some space $\productspace$.
The {\em dependency graph} $\Gamma({\bf Z})=(V,E)$ of ${\bf Z}$ is such that: 
\begin{itemize}
\item the set of vertices $V$ of $\Gamma({\bf Z})$ is $V=[m]$;
\item $(i,j)\not\in E$ (there is no edge between $i$ and $j$)
$\Leftrightarrow$ $Z_i$ and $Z_j$ are independent.
\end{itemize}
\end{definition}
\begin{definition}[Fractional Covers, \cite{schreinerman97fractional}]
\label{def:fc}
Let $\Gamma=(V,E)$ be an undirected graph, with $V=[m]$.
\begin{itemize}
\item $C\subseteq V$ is {\em independent} if the vertices in $C$ are independent (no two vertices in $C$ are connected).
\item ${\bf C}=\{C_j\}_{j=1}^{n}$, with $C_j\subseteq V$, is a {\em proper cover} of $V$ if each $C_j$ is independent and $\union_{j=1}^{n}C_j=V$. It is {\em exact} if $\bfC$ is a partition of $V$. The size of ${\bf C}$ is $n$.
\item ${\bf C}=\{(C_j,\omega_j)\}_{j=1}^{n}$, with $C_j\subseteq V$ and $\omega_j\in[0,1]$, is a {\em proper exact fractional cover} of $V$ if each $C_j$ is independent and   $\forall i\in V$, $\sum_{j=1}^{n}\omega_j\indicator_{i\in C_j}=1$; $\omega(\bfC)=\sum_{j=1}^n\omega_i$ is the {\em chromatic weight} of $\bfC$.
\item The (fractional) chromatic number $\chi(\Gamma)$ ($\chi^*(\Gamma)$) is the minimum size (chromatic weight) over all proper exact (fractional) covers of $\Gamma$
\end{itemize}
\end{definition}
A cover is a fractional cover such that all the weights $\omega_i$ are
equal to $1$ (and all the results we state for fractional
covers apply to the case of covers). If $n$ is the size of a cover, it
means that the nodes of the graph at hand can be colored with $n$
colors in a way such that no two adjacent nodes receive the same
color.

The problem of computing the (fractional) chromatic number of a graph
is \np-hard \citep{schreinerman97fractional}. However, for some particular graphs as those that come from the
settings we study in Section~\ref{sec:examples}, this number can be
evaluated precisely. If it cannot be evaluated, it can be upper bounded using the following property.
\begin{property}[\cite{schreinerman97fractional}]
\label{prop:fc}
Let $\Gamma=(V,E)$ be a graph. Let $c(\Gamma)$ be the {\em clique
  number} of $\Gamma$, i.e. the order of the largest clique in
$\Gamma$. Let $\Delta(\Gamma)$ be the maximum degree of a vertex in
$\Gamma$. We have the following inequalities:
$$1\leq c(\Gamma)\leq\chi^*(\Gamma)\leq\chi(\Gamma)\leq\Delta(\Gamma)+1.$$
In addition,
$1=c(\Gamma)=\chi^*(\Gamma)=\chi(\Gamma)=\Delta(\Gamma)+1$ {\em if and
  only if} $\Gamma$ is totally disconnected. \end{property}

If ${\bf Z}=\{Z_i\}_{i=1}^m$ is a set of random variables over $\productspace$ then 
a (fractional) proper cover of $\Gamma({\bf Z})$, splits {\bf Z} into subsets
of independent random variables. This is a crucial feature to establish our results. In addition, we can see $\chi^*(\Gamma(\bfZ))$ and $\chi(\Gamma(\bfZ))$
as measures of the amount of dependencies within $\bfZ$.

The following lemma (Lemma 3.1 in \citep{janson04large}) will be very useful in the following.
\begin{lemma}
\label{lem:efc}
If ${\bf C}=\{(C_j,\omega_j)\}_{j=1}^{n}$ is an exact fractional cover
of $\Gamma=(V,E)$, with $V=[m]$, then
$$\forall{\bf t}\in\realset^{m},\;\sum_{i=1}^mt_i=\sum_{j=1}^n\omega_j\sum_{k\in C_j}t_k.$$
In particular, $m=\sum_{j=1}^n\omega_j|C_j|$. \end{lemma}

\subsection{Chromatic \pac-Bayes Bounds}
We now provide new \pac-Bayes bounds for classifiers trained on
samples $\bfZ$ drawn from distributions $\bfD_m$ where dependencies
exist. We assume these dependencies are fully determined by $\bfD_m$
and we define the dependency graph $\Gamma(\bfD_m)$ of $\bfD_m$ to be
$\Gamma(\bfD_m)=\Gamma(\bfZ)$. As said before, the marginal
distributions of $\bfD_m$ along each coordinate are the same and are
equal to some distribution $D$.

We introduce additional notation.
 $\pefc(\bfD_m)$ is the set of proper exact fractional covers
of $\Gamma(\bfD_m)$. Given a cover $\bfC=\{(C_j,\omega_j)\}_{j=1}^n\in\pefc(\bfD_m)$, we use the following notation:
\begin{itemize}
\item  $\bfZ^{(j)}=\{Z_k\}_{k\in C_j}$;
\item $\bfD_m^{(j)}$, the distribution of $\bfZ^{(j)}$: it is equal to $D^{|C_j|}=\otimes_{i=1}^{|C_j|}D$ ($C_j$ is independent);
\item $\bfalpha=(\alpha_j)_{1\leq j\leq n}$ with $\alpha_j=\omega_j/\omega(\bfC)$: we have $\alpha_j\geq 0$ and $\sum_j\alpha_j=1$;
\item $\bfpi=(\pi_j)_{1\leq j\leq n}$, with $\pi_j=\omega_j|C_j|/m$: we have $\pi_j\geq 0$ and $\sum_j\pi_j=1$ (cf. Lemma~\ref{lem:efc}).
\end{itemize}
In addition, $\bfP_n$ and $\bfQ_n$ denote distributions over
$\family^n$, $P_n^j$ and $Q_n^j$ are the marginal distributions of
$\bfP_n$ and $\bfQ_n$ with respect to the $j$th coordinate,
respectively.

We can now state our main results. 
\begin{theorem}[Chromatic
  \pac-Bayes Bound (I)]\label{th:cpbb1}
$\forall \bfD_m$, $\forall\family$,
  $\forall\delta\in(0,1]$, $\forall
  \bfC=\{(C_j,\omega_j)\}_{j=1}^n\in\pefc(\bfD_m)$, $\forall \bfP_n$,
  with probability at least $1-\delta$ over the random draw of
  $\bfZ\sim\bfD_m$, the following holds:
  \begin{equationsize}{\normalsize}
    \forall\bfQ_n,\;\kl(\bar{e}_{\bfQ_n}(\bfZ)||e_{\bfQ_n})\leq\frac{\omega}{m}\left[\sum_{j=1}^n\alpha_j\KL(Q_n^j||P_n^j)+\ln\frac{m+\omega}{\delta
        \omega}\right],\label{eq:cpbb1}
  \end{equationsize}
  where $\omega$ stands for $\omega(\bfC)$, and
\begin{align*}
\bar{e}_{\bfQ_n}(\bfZ)&:=\sum_{j=1}^{n}\pi_j\expectation_{h\sim  Q_n^j}\hat{\risk}(h,\bfZ^{(j)}),\\
  e_{\bfQ_n}&:=\expectation_{\bfZ\sim\bfD_m}\bar{e}_{\bfQ_n}(\bfZ).
\end{align*}
\end{theorem}
\begin{proof} Deferred to Section~\ref{sec:proof}.\end{proof}
We would like to emphasize that the same type of result -- using the
same proof techniques -- can be
obtained if simple (i.e. not exact nor proper) fractional covers are
considered. However, as we shall see, the `best' (in terms of
tightness) bound is achieved for covers from the set of proper exact
fractional covers, and this is the reason why we have stated
Theorem~\ref{th:cpbb1} with a restriction to this particular set of covers.

 The empirical quantity $\bar{e}_{\bfQ_n}(\bfZ)$ is a weighted average of the
  empirical errors on $\bfZ^{(j)}$ of Gibbs classifiers with
  respective distributions $Q_n^j$. 
The following proposition characterizes $e_{\bfQ_n}=\expectation_{\bfZ\sim
  \bfD_m}\bar{e}_{\bfQ_n}(\bfZ)$. 
\begin{proposition} \label{prop:mixture}
$\forall \bfD_m$, $\forall\family$, $\forall
  \bfC=\{(C_j,\omega_j)\}_{j=1}^n\in\pefc(\bfD_m)$, $\forall \bfQ_n$:
  $e_{\bfQ_n}=\expectation_{\bfZ\sim \bfD_m}\bar{e}_{\bfQ_n}(\bfZ)$ is the
  error of the Gibbs classifier based on the mixture of distributions
  $Q^{\bfpi}=\sum_{j=1}^n\pi_jQ_n^j$. 
\end{proposition}
\begin{proof} From the definition of $\bfpi$, $\pi_j\geq 0$ and $\sum_{j=1}^n\pi_j=1$. Thus,
  \begin{alignsize*}{\normalsize}
    \expectation_{\bfZ\sim \bfD_m}\bar{e}_{\bfQ_n}(\bfZ)&=\expectation_{\bfZ\sim \bfD_m}\sum_{j}\pi_j\expectation_{h\sim  Q_n^j}\hat{\risk}(h,\bfZ^{(j)})\\
    &=\sum_{j}\pi_j\expectation_{h\sim Q_j}\expectation_{\bfZ^{(j)}\sim \bfD_m^{(j)}}\hat{\risk}(h,\bfZ^{(j)}) \tag{marginalization}\\
    &=\sum_{j}\pi_j\expectation_{h\sim Q_n^{j}}\risk(h) \tag{$\expectation_{\bfZ^{(j)}\sim \bfD_m^{(j)}}\hat{\risk}(h,\bfZ^{(j)})=R(h),\forall j$}\\
    &=\expectation_{h\sim\pi_1Q_n^1+\ldots+\pi_jQ_n^j}\risk(h)=\expectation_{h\sim
      Q^{\bfpi}}\risk(h). 
  \end{alignsize*}
Where, in the third line, we have used the fact that the variables in
$\bfZ^{(j)}$ are identically distributed (by assumption, they are IID).
\end{proof}

\begin{remark}\label{rem:algo} The prior $\bfP_n$
  and the posterior $\bfQ_n$ enter into play in
  Proposition~\ref{prop:mixture} and Theorem~\ref{th:cpbb1} through
  their marginals only. This  advocates for the following learning scheme.
  Given a cover and a (possibly factorized) prior $\bfP_n$, look for a
  factorized posterior $\bfQ_n=\otimes_{j=1}^nQ_j$ such that each
  $Q_j$ independently minimizes the usual \iid \pac-Bayes bound given
  in Theorem~\ref{th:pbiid} on each $\bfZ^{(j)}$. Then make
  predictions according to the Gibbs classifier defined with respect
  to $Q^{\bfpi}=\sum_j\pi_jQ_j$. \end{remark}

The following theorem gives a result that readily applies without
choosing a specific cover. 
\begin{theorem}[Chromatic \pac-Bayes Bound
  (II)] \label{th:cpbb2}
  $\forall \bfD_m$, $\forall\family$,
  $\forall\delta\in(0,1]$, $\forall P$, with probability at least
  $1-\delta$ over the random draw of $\bfZ\sim\bfD_m$, the following
  holds
  \begin{equationsize}{\normalsize}
    \forall Q,\;\kl(\hat{e}_Q(\bfZ)||e_Q)\leq\frac{\chi^*}{m}\left[\KL(Q||P)+\ln\frac{m+\chi^*}{\delta\chi^*}\right],\label{eq:cpbb2}
\end{equationsize}
  where $\chi^*$ is the fractional chromatic number of $\Gamma(\bfD_m)$, and where $\hat{e}_Q(\bfZ)$ and $e_Q$ are as in~(\ref{eq:gibbs}).
\end{theorem}
\begin{proof}
  This theorem is just a particular case of Theorem~\ref{th:cpbb1}.
  Assume that $\bfC=\{(C_j,\omega_j)\}_{j=1}^n\in\pefc(\bfD_m)$ such
  that $\omega(C)=\chi^{*}(\Gamma(\bfD_m))$,
  $\bfP_n=\otimes_{j=1}^nP=P^n$ and $\bfQ_n=\otimes_{j=1}^nQ=Q^n$, for
  some $P$ and $Q$.

  For the right-hand side of~(\ref{eq:cpbb2}), it directly comes that
  $$\sum_j\alpha_j\KL(Q_n^j||P_n^j)=\sum_j{\alpha_j}
  \KL(Q||P)=\KL(Q||P).$$ It then suffices to show that
  $\bar{e}_{\bfQ_n}(\bfZ)=\hat{e}_Q(\bfZ)$:
  \begin{alignsize*}{\normalsize}
    \bar{e}_{\bfQ_n}(\bfZ)&=\sum_j\pi_j\expectation_{h\sim Q_n^j}\hat{\risk}(h,\bfZ^{(j)})=\sum_j\pi_j\expectation_{h\sim Q}\hat{\risk}(h,\bfZ^{(j)})\\
    &=\frac{1}{m}\sum_j\omega_j|C_j|\expectation_{h\sim Q}\frac{1}{|C_j|}\sum_k r(h,Z_k)\tag{$\pi_j=\frac{\omega_j|C_j|}{m},\forall j$}\\
    &=\expectation_{h\sim Q}\frac{1}{m}\sum_j\omega_j\sum_kr(h,Z_k)\\
    &=\expectation_{h\sim Q}\frac{1}{m}\sum_{i}r(h,Z_i)\tag{cf. Lemma~\ref{lem:efc}}\\
    &=\expectation_{h\sim Q}\hat{\risk}(h,\bfZ)=\hat{e}_Q(\bfZ). 
  \end{alignsize*}
\end{proof}

A few comments are in order.
\begin{itemize}
\item A $\chi^*$ worsening.
 This theorem says that even in the case of non \iid data, a
  \pac-Bayes bound very similar to the \iid \pac-Bayes
  bound~(\ref{eq:pbiid}) can be stated, with a worsening (since
  $\chi^*\geq 1$) proportional to $\chi^*$, i.e proportional to the
  amount of dependencies in the data. In
  addition, the new \pac-Bayes bounds is valid with any priors and
  posteriors, without the need for these distributions to depend on the chosen cover (as is the case with the
  more general
  Theorem~\ref{th:cpbb1}).
\item $\chi^*$: the optimal constant.
  Among all elements of $\pefc(\bfD_m)$, $\chi^*$
  is the best constant achievable in terms of the tightness of the
  bound~\eqref{eq:cpbb2} on $e_Q$: getting an optimal coloring gives
  rise to an `optimal' bound. Indeed, it suffices to observe that the
  right-hand side of~\eqref{eq:cpbb1} is decreasing with respect to
  $\omega$ when all $Q^{j}_n$ are identical (we let the reader check that).
 As $\chi^*$ is the smallest chromatic weight, it gives the tightest
  bound. 
\item $\Gamma(\bfD_m)$ vs. induced subgraphs.
If $\bfs\subseteq[m]$ and $\bfZ_{\bfs}=\{Z_s:s\in \bfs\}$, it is obvious that Theorem~\ref{th:cpbb2} holds for $|\bfs|$-samples drawn from the marginal distribution
$\bfD_{\bfs}$ of $\bfZ_{\bfs}$. Considering only $\bfZ_{\bfs}$ amounts to working with the subgraph $\Gamma(\bfD_{\bfs})$ of $\Gamma(\bfD_m)$ induced by the vertices in $\bfs$: this might provide a better bound in situations where $\chi^*(\bfD_{\bfs})/|\bfs|$ is smaller than $\chi^*(\bfD_m)/m$ (this is not guaranteed, however, because the empirical error $\hat{e}_Q(\bfZ_{\bfs})$ computed on $\bfZ_{\bfs}$ might be larger than $\hat{e}_Q(\bfZ)$). To see this, consider a graph $\Gamma_{\text{1-edge}}=(V,E)$ of $m$ vertices where $|E|=1$, i.e. there are only two nodes, say $u$ and $v$, that are connected  (see Figure~\ref{fig:oneedge}). The fractional chromatic number $\chi^*_{\text{1-edge}}$ of $\Gamma_{\text{1-edge}}$ is $2$ ($u$ and $v$ must use distinct colors) while the (fractional) chromatic number $\chi^*_u$ of the subgraph $\Gamma_u$ of $\Gamma_{\text{1-edge}}$ obtained by removing $u$ is $1$: $\chi^{*}_{\text{1-edge}}$ is twice as big as $\chi^{*}_u$ while the number of nodes only differ by $1$ and, for large $m$, this ratio roughly carries over for $\chi^{*}_{\text{1-edge}}/m$ and $\chi^{*}_u/(m-1)$.
\end{itemize}

\begin{figure}[t]
\centering
\subfloat[$\Gamma_{\text{1-edge}}$]{
\includegraphics[width=0.38\textwidth]{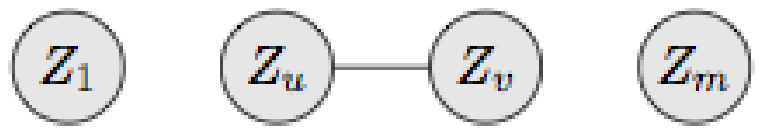}
}
\hspace{1.5cm}
\subfloat[$\Gamma_{u}$]{
\includegraphics[width=0.38\textwidth]{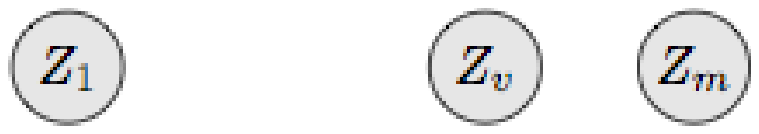}
}
\caption{$\Gamma_u$ is the subgraph induced by $\Gamma_{\text{1-edge}}$ --~which contains only one edge, between $u$ and $v$~-- when $u$ is removed: it might be preferable to consider the distribution corresponding to $\Gamma_u$ in Theorem~\ref{th:cpbb2} instead of the distribution defined wrt $\Gamma_{\text{1-edge}}$, since $\chi^*(\Gamma_{\text{1-edge}})=2$ and $\chi^*(\Gamma_u)=1$ (see text for detailed comments).\label{fig:oneedge}}
\end{figure}

This last comment outlines that considering a subset of $\bfZ$, or,
equivalently, a subgraph of $\Gamma(\bfD_m)$, in \eqref{eq:cpbb2}, might provide a better
generalization bound. However, it is assumed that the choice of the subgraph is done
{\em before} computing the bound: the bound does only hold
 with probability $1-\delta$ for the chosen subgraph. To alleviate this
and provide a bound that takes advantage of several induced subgraphs, 
we have the following proposition:

\begin{proposition}  Let $\{m\}^{\# k}$ denote
  $\{\bfs:\bfs\subseteq[m],|\bfs|= m- k\}$.
 $\forall \bfD_m$, $\forall\family$, $\forall k\in[m]$,
  $\forall\delta\in(0,1]$, $\forall P$, with probability at least
  $1-\delta$ over the random draw of $\bfZ\sim\bfD_m$: $\forall Q,$
  \begin{equationsize}{\small}
e_Q\leq\min_{\bfs\in\{m\}^{\# k}}\left\{\hat{e}_{Q}(\bfZ_\bfs)+\kl^{-1}\left(\hat{e}_{Q}(\bfZ_\bfs),\frac{\chi^*_{\bfs}}{|\bfs|}\left[\KL(Q||P)+\ln\frac{|\bfs|+\chi^*_{\bfs}}{\chi^*_{\bfs}}+ \ln\binom{m}{k}+\ln\frac{1}{\delta}\right]\right)\right\}.\label{eq:pbiidsubgraph}
\end{equationsize}
  where $\chi^*_{\bfs}$ is the fractional chromatic number of $\Gamma(\bfD_{\bfs})$, and where $\hat{e}_{Q}(\bfZ_\bfs)$ is the empirical error of the Gibbs classifier
$g_Q$ on $\bfZ_\bfs$, that is: $\hat{e}_{Q}(\bfZ_\bfs)=\expectation_{h\sim Q}\hat\risk(h,\bfZ_\bfs)$.
\end{proposition}
\begin{proof}
Simply apply the union bound to equation~\eqref{eq:cpbb2} of Theorem~\ref{th:cpbb2}:
for fixed $k$, there are $\binom{m}{m-k}=\binom{m}{k}$ subgraphs and
using $\delta/\binom{m}{k}$ makes the bound hold with probability
$1-\delta$ for all possible $\binom{m}{k}$ subgraphs
(simultaneously). Making use of the form \eqref{eq:pbiidinvkl} gives
the result.
\end{proof}

This bound is particularly useful when, for some small $k$,
there exists a subset $\bfs\subseteq \{m\}^{\# k}$ such that the induced subgraph
$\Gamma(\bfD_\bfs)$, which has $k$ fewer nodes than
$\Gamma(\bfD_m)$, has a fractional chromatic number $\chi^{*}_{\bfs}$ that is
smaller than $\chi^{*}(\bfD_m)$ (as is the case with
the graph $\Gamma_{\text{1-edge}}$ of Figure~\ref{fig:oneedge}, where
$k=1$). Obtaining a similar result that holds for subgraphs associated with
sets $\bfs$ of sizes {\em larger or equal} to $m-k$ is possible by replacing
$\ln\binom{m}{k}$ with $\ln\sum_{\kappa=0}^k\binom{m}{\kappa}$ in the
bound (in that case, $k$ should be kept small enough with respect to
$m$, e.g. $k=\mathcal{O}_m(1)$,  to ensure that
the resulting bound still goes down to zero when $m\rightarrow\infty$).

\subsection{On the Relevance of  Fractional Covers}
One may wonder whether using the fractional cover framework is the
only way to establish a result similar to the one provided by
Theorem~\ref{th:cpbb1}. Of course, this is not the case and one
may imagine other ways of deriving closely related results without
mentioning the idea of fractional/cover coloring. (For instance, one may
manipulate subsets of independent
variables, assign weights to these subsets without referring to
fractional covers, and arrive at results that are comparable to
ours.)

However, if we assume that singling out independent sets of 
variables is the cornerstone of dealing with interdependent random
variables, we find it
enlightning to cast our approach within the rich and well-studied fractional cover/coloring
framework. On the one hand, our objective of deriving tight bounds
amounts to finding a decomposition of the set of random variables at
hand into {\em few and large} independent subsets and taking the graph
theory point of view, this obviously corresponds
to a problem of graph coloring. Explicitly using the fractional
cover/coloring argument allows us to directly benefit from the wealth
of related results, such as Property 1 or, for instance, approaches as
to how compute a cover or approximate the fractional chromatic number (e.g., linear
programming). On the other hand, from a technical point of view, making use of the fractional cover
argument allows us to preserve the simple structure of the proof of
the classical IID PAC-Bayes bound to derive Theorem~\ref{th:cpbb1}. 

To summarize, the richness of the results on graph (fractional)
coloring provides us with elegant tools to deal with a natural
representation of the dependencies that may occur within a set of
random variables. In addition, and as showed in this article, it is
possible to 
seamlessly take advantage of these tools
in the PAC-Bayesian framework (and
probably in other bound-related frameworks).


\subsection{Proof of Theorem~\ref{th:cpbb1}}
\label{sec:proof}
A proof in three steps, following the lines of the proofs given
by~\cite{seeger02pac} and~\cite{langford05tutorial} for the \iid
\pac-Bayes bound, can be provided.

\begin{lemma}
\label{lem:l1}
$\forall \bfD_m$, $\forall\delta\in(0,1]$, $\forall\bfC=\{(C_j,\omega_j)\}_{j=1}^n$, $\forall\bfP_n$ distribution over $\family^n$, with probability at least $1-\delta$ over the random draw of $\bfZ\sim\bfD_m$, the following holds (here, $\omega$ stands for $\omega(\bfC)$)
\begin{equation}\expectation_{\bfh\sim\bfP_n}\sum_{j=1}^n\alpha_je^{|C_j|\kl(\hat{\risk}(h_j,\bfZ^{(j)})||\risk(h_j))}\leq \frac{m+\omega}{\delta\omega},
\label{eq:l1}
\end{equation}
where $\bfh=(h_1,\ldots,h_n)$ is a random vector of hypotheses.
\end{lemma}
\begin{proof}
We first observe the following:
\begin{alignsize*}{\normalsize}
\expectation_{\bfZ\sim\bfD_m}\sum_j\alpha_je^{|C_j|\kl(\hat{\risk}(h_j,\bfZ^{(j)})||\risk(h_j))}
&=\sum_j\alpha_j\expectation_{\bfZ^{(j)}\sim \bfD_m^{(j)}}e^{|C_j|\kl(\hat{\risk}(h,\bfZ^{(j)})||\risk(h))}\\
&\leq\sum_j\alpha_j(|C_j|+1)\tag{Lemma~\ref{lem:entropy}, Appendix}\\
&=\frac{1}{\omega}\sum_j\omega_j(|C_j|+1)\\
&=\frac{m+\omega}{\omega},\tag{Lemma~\ref{lem:efc}}
\end{alignsize*}
where using Lemma~\ref{lem:entropy} is made possible by the fact that $\bfZ^{(j)}$ is an \iid sample. Therefore,
\begin{equationsize*}{\normalsize}
\expectation_{\bfZ\sim\bfD_m}\expectation_{\bfh\sim\bfP_n}\sum_{j=1}^n\alpha_je^{|C_j|\kl(\hat{\risk}(h_j,\bfZ^{(j)})||\risk(h_j))}\leq \frac{m+\omega}{\omega}.
\end{equationsize*}
According to Markov's inequality (Theorem~\ref{th:markov}, Appendix),
$$\proba_{\bfZ}\left(\expectation_{\bfh\sim\bfP_n}\sum_j\alpha_je^{|C_j|\kl(\hat{\risk}(h_j,\bfZ^{(j)})||\risk(h_j))}\geq \frac{m+\omega}{\omega\delta}\right)\leq\delta.$$
\end{proof}

\begin{lemma}
\label{lem:l2}
$\forall\bfD_m$, $\forall\bfC=\{(C_j,\omega_j)\}_{j=1}^n$, $\forall\bfP_n$, $\forall\bfQ_n$, with probability at least $1-\delta$ over the random draw of $\bfZ\sim\bfD_m$, the following holds
\begin{equationsize}{\normalsize}
\frac{m}{\omega}\sum\nolimits_{j=1}^n\pi_j\expectation_{h\sim Q_n^j}\kl(\hat{\risk}(h,\bfZ^{(j)})||\risk(h))
\leq\sum\nolimits_{j=1}^n\alpha_j\KL(Q_n^j||P_n^j)+\ln\frac{m+\omega}{\delta\omega}.\label{eq:l2}
\end{equationsize}
\end{lemma}
\begin{proof}
It suffices to use Jensen's inequality (Theorem~\ref{th:jensen}, Appendix) with $\ln$ and the fact that
$\expectation_{X\sim P}f(X)=\expectation_{X\sim Q}\frac{P(X)}{Q(X)}f(X)$, for
all $f, P, Q$. Therefore, $\forall\bfQ_n$:
\begin{alignsize*}{\normalsize}
\ln \expectation_{\bfh\sim\bfP_n}&\sum_j\alpha_je^{|C_j|\kl\left(\hat{\risk}(h_j,\bfZ^{(j)})||\risk(h_j)\right)}
=\ln\sum_j\alpha_j\expectation_{h\sim P_n^j}e^{|C_j|\kl\left(\hat{\risk}(h,\bfZ^{(j)})||\risk(h)\right)}\\
&=\ln\sum_j\alpha_j\expectation_{h\sim Q_n^j}\frac{P_n^j(h)}{Q_n^j(h)}e^{|C_j|\kl\left(\hat{\risk}(h,\bfZ^{(j)})||\risk(h)\right)}\\
&\geq\sum_j\alpha_j\expectation_{h\sim Q_n^j}\ln\left[\frac{P_n^j(h)}{Q_n^j(h)}e^{|C_j|\kl\left(\hat{\risk}(h,\bfZ^{(j)})||\risk(h)\right)}\right]\tag{Jensen's inequality}\\
&=-\sum_j\alpha_j\KL(Q_n^j||P_n^j) +\sum_j\alpha_j|C_j|\expectation_{h\sim Q_n^j}\kl\left(\hat{\risk}(h,\bfZ^{(j)})||\risk(h)\right)\\
&=-\sum_j\alpha_j\KL(Q_n^j||P_n^j) +\frac{m}{\omega}\sum_j\pi_j\expectation_{h\sim Q_n^j}\kl\left(\hat{\risk}(h,\bfZ^{(j)})||\risk(h)\right).
\end{alignsize*}
Lemma~\ref{lem:l1} then gives the result. 
\end{proof}

\begin{lemma}
\label{lem:l3}
$\forall\bfD_m$, $\forall\bfC=\{(C_j,\omega_j)\}_{j=1}^n$,  $\forall\bfQ_n,$, the following holds
\begin{equationsize*}{\normalsize}
\frac{m}{\omega}\sum\nolimits_{j=1}^n\pi_j\expectation_{h\sim Q_n^j}\kl(\hat{\risk}(h,\bfZ^{(j)})||\risk(h))\geq\kl(\bar{e}_Q||e_Q).
\end{equationsize*}
\end{lemma}
\begin{proof}
This simply comes from the convexity of $\kl(x,y)$ in $(x,y)$ (Lemma~\ref{lem:kl}, Appendix). This, in combination with Lemma~\ref{lem:l2}, closes the proof of Theorem~\ref{th:cpbb1}.
\end{proof}

 
\section{Applications}
\label{sec:examples}
In this section, we provide instances of Theorem~\ref{th:cpbb2} for various settings; amazingly,
they alllow us to easily derive \pac-Bayes generalization bounds for problems such as ranking
and learning from stationary $\beta$-mixing processes. The theorems we provide here are all
new \pac-Bayes bounds for different non-\iid settings.
\begin{figure}
\centering
\subfloat[\iid data\label{fig:iid}]{
\includegraphics[width=0.23\textwidth]{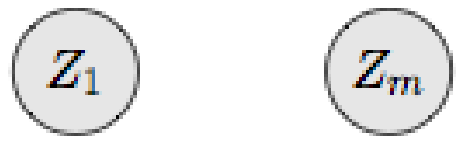}
}
\hspace{2cm}
\subfloat[Bipartite ranking data\label{fig:ranking}]{
\includegraphics[width=0.4\textwidth]{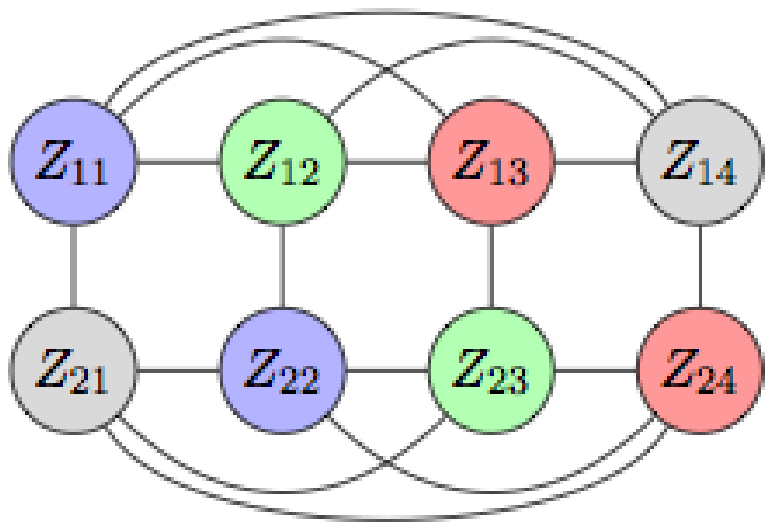}
}
\caption{Dependency graphs for different settings described in section~\ref{sec:examples}. Nodes of the
same color are part of the same cover element;  hence, they are probabilistically independent. (a) When the
data are \iid, the dependency graph is disconnected and the fractional number is $\chi^*=1$; (b)
a dependency graph obtained for bipartite ranking from a sample of 4 positive and 2 negative instances: $\chi^*=4$.\label{fig:examples}}
\end{figure}
\subsection{IID Case}
The first case we are interested in is the \iid setting.
In this case, the training sample $\bfZ=\{(X_i,Y_i)\}_{i=1}^m$
is distributed according to $\bfD_m=D^m$ and the fractional chromatic
number of $\Gamma(\bfD_m)$ is $\chi^*=1$, since the dependency graph, depicted in Figure~\ref{fig:iid} is totally
disconnected (see Property~\ref{prop:fc}).  Plugging in this value of
$\chi^*$ in the bound of Theorem~\ref{th:cpbb2} gives the \iid
\pac-Bayes bound of Theorem~\ref{th:pbiid}. This emphasizes
the fact that the standard \pac-Bayes bound is a special case of our more general results.

\subsection{General Ranking and Connection to U-Statistics}
Here, the learning problem of interest is the
following. ${D}$ is a distribution over
${\inputspace}\times{\outputspace}$ with
${\outputspace}=\realset$ and one looks for a ranking rule
$h\in\realset^{\inputspace\times\inputspace}$
that minimizes the {\em ranking risk} $R^{\rank}(h)$ defined as:
\begin{equation}
  R^{\rank}(h):=\proba_{\ontop{({X},{Y})\sim{D}}{({X}',{Y}')\sim{D}}}(({Y}-{Y}')h({X},{X}')<0).\label{eq:rankingrisk}
\end{equation}
For a random pair $({X},{Y})$, ${Y}$ can be
thought of as a score that allows one to rank objects: given two pairs
$({X},{Y})$ and $({X}',{Y}')$,
${X}$ has a higher rank (or is `better') than ${X}'$
if ${Y}>{Y}'$. The ranking rule $h$ predicts
${X}$ to be better than ${X}'$ if
$\sign(h({X},{X}'))=1$ and conversely. 
The objective of
learning is to produce a rule $h$ that makes as few misrankings as
possible, as measured by~\eqref{eq:rankingrisk}.  Given a finite \iid
(according to ${D}$) sample
$\bfS=\{({X}_i,{Y}_i)\}_{i=1}^{\ell}$ an unbiased
estimate of $R^{\rank}(h)$ is $\hat{\risk}^{\rank}(h,\bfS)$, with:
\begin{equation}
  \hat{R}^{\rank}(h,\bfS):=\frac{1}{\ell(\ell-1)}\sum_{i\neq j}\indicator_{({Y}_{i}-{Y}_j)h({X}_i,{X}_j)<0}=\frac{1}{\ell(\ell-1)}\sum_{i\neq j}\indicator_{{Y}_{ij}h({X}_i,{X}_j)<0},\label{eq:emprankingrisk}
\end{equation}
where ${Y}_{ij}:=({Y}_i-{Y}_j)$. A natural question
is to bound the ranking risk for any learning rule $h$ given $\bfS$,
where the difficulty is that \eqref{eq:emprankingrisk} is a sum of
identically but not independently random variables, namely the
variables $\indicator_{Y_{ij}h({X}_i,{X}_j)}$.  

Let us define
$X_{ij}:=({X}_i,{X}_j)$, $Z_{ij}:=(X_{ij},Y_{ij})$, 
and $\bfZ:=\{Z_{ij}\}_{i\neq j}$. We note that the number $\ell$ of training data suffices to determine the structure
of the dependency graph $\Gamma_{\rank}$ of $\bfZ$ and its distribution, which
we denote $\bfD_{\ell(\ell-1)}$. Henceforth, we are clearly in the framework
for the application of the chromatic \pac-Bayes bounds defined in the previous section.
In particular, to instantiate Theorem~\ref{th:cpbb2} to the present ranking problem, we simply need to have at hand
the value $\chi^{*}_{\rank}$ --~or an upper bound thereof~-- of the fractional chromatic number of $\Gamma_{\rank}$.
We claim that $\chi^*_{\rank}\leq \ell(\ell-1)/\lfloor\ell/2\rfloor$ where $\lfloor x\rfloor$ is the largest integer less than or equal to $x$. We provide the following new \pac-Bayes bound for the ranking risk:
\begin{theorem}[Ranking \pac-Bayes bound]
\label{th:pbranking}
$\forall {D}$ over ${\inputspace}\times{\outputspace}$, $\forall{\family}\subseteq\realset^{{\inputspace}\times{\inputspace}}$, $\forall\delta\in(0,1]$, $\forall {P}$, with
probability at least $1-\delta$ over the random draw of $\bfS\sim {D}^\ell$, the following holds
  \begin{equationsize}{\small}
    \forall {Q},\;\kl(\hat{e}_Q^{\rank}(\bfS)||e_Q^{\rank})\leq\frac{1}{\lfloor\ell/2\rfloor}\left[\KL({Q}||{P})+\ln\frac{\lfloor\ell/2\rfloor+1}{\delta}\right],\label{eq:pbranking}
\end{equationsize}
where 
\begin{align*}
\hat{e}_{{Q}}^{\rank}(\bfS)&:=\expectation_{h\sim {Q}}\hat{R}^{\rank}(h,\bfS)\\
e_{{Q}}^{\rank}&:=\expectation_{\bfS\sim D^{\ell}}\hat{e}_{{Q}}^{\rank}(\bfS).
\end{align*}
\end{theorem}
\begin{proof}
We essentially need to prove our claim on the bound on $\chi^*_{\rank}$.
To do so, we consider a fractional cover of $\Gamma_{\rank}$ motivated by the theory of
U-statistics \citep{hoeffding48class,hoeffding63probability}. $\hat{\risk}(h,\bfS)$
is indeed a U-statistics of order 2 and it might be rewritten as a sum of \iid blocks as follows
\begin{align*}
\hat{\risk}(h,\bfS)&=\frac{1}{\ell(\ell-1)}\sum_{i\neq j}r(h,Z_{ij})=\frac{1}{\ell !}\sum_{\sigma\in\Sigma_\ell}\frac{1}{\lfloor\ell/2\rfloor}\sum_{i=1}^{\lfloor\ell/2\rfloor}r\left(h,Z_{\sigma(i)\sigma(\lfloor\ell/2\rfloor+i)}\right),
\end{align*}
where $\Sigma_{\ell}$ is the set of permutations over $[\ell]$.
The innermost sum is obviously a sum of \iid random variables as no two summands share the same indices.

A proper exact fractional cover $\bfC_\rank$ can be derived from this decomposition as\footnote{Note that the cover defined here 
considers elements $C_\sigma$ containing random variables themselves instead of
their indices. This abuse of notation is made for sake of readability.}
$$\bfC_{\rank}:=\left\{\left(C_\sigma:=\left\{Z_{\sigma(i)\sigma(\lfloor\ell/2\rfloor+i)}\right\}_{i=1}^{\lfloor \ell/2\rfloor},\omega_\sigma:=\frac{1}{(\ell-2)!\lfloor\ell/2\rfloor}\right)\right\}_{\sigma\in\Sigma_\ell}.$$
Indeed, as remarked before, each $C_\sigma$ is an independent set and
 each random variable $Z_{pq}$ for $p\neq q$, appears in exactly $(\ell-2)!\times \lfloor\ell/2\rfloor$ sets $C_\sigma$ (for $i$ fixed,
the number of permutations $\sigma$ such that $\sigma(i)=p$ and $\sigma(\lfloor\ell/2\rfloor+i)=q$ is
equal to $(\ell-2)!$, i.e. the number of permutations on $\ell-2$ elements; as $i$ can take $\lfloor\ell/2\rfloor$ values, this gives the result). Therefore, $\forall p,q, p\neq q$:
$$\sum_{\sigma\in\Sigma_{\ell}}\omega_\sigma\indicator_{Z_{pq}\in C_\sigma}=\frac{1}{(\ell-2)!\lfloor\ell/2\rfloor}\sum_{\sigma\in\Sigma_{\ell}}\indicator_{Z_{pq}\in C_\sigma}=\frac{1}{(\ell-2)!\lfloor\ell/2\rfloor}\times (\ell-2)!\lfloor\ell/2\rfloor=1,$$
which proves that $\bfC_{\rank}$ is a proper exact fractional cover. Its weight $\omega(\bfC_{\rank})$ is 
$$\omega(\bfC_{\rank})=\ell!\times\omega_\sigma=\frac{\ell(\ell-1))}{\lfloor\ell/2\rfloor}.$$
Hence, from the definition of $\chi^*_{\rank}$, $$\chi^*_{\rank}\leq \frac{\ell(\ell-1))}{\lfloor\ell/2\rfloor}.$$

The theorem follows by an instantiation of Theorem~\ref{th:cpbb2} with $m:=\ell(\ell-1)$ and the bound
on $\chi^*_{\rank}$ we have just proven.
\end{proof}
To our knowledge, this is the first \pac-Bayes bound on the ranking risk, while a Rademacher-complexity
based analysis was given by \cite{clemencon08ranking}. In the proof, we have used arguments from
the analysis of U-processes, which allow us to easily derive a convenient fractional cover of the
dependency graph of $\bfZ$. Note however that our framework still applies even if not all
the $Z_{ij}$'s are known, as required if an analysis based on U-processes is undertaken. This is
particularly handy in practical situations where one may only be given the values $Y_{ij}$ --~but {\em not}
the values of $Y_i$ and $Y_j$~-- for
a limited number of $(i,j)$ pairs (and not all the pairs).

An interesting question is to know how the so-called Hoeffding decomposition
used by \cite{clemencon08ranking} to establish fast rates of convergence for empirical ranking
risk minimizers could be used to draw possibly tighter \pac-Bayes bounds. This would imply
being able to appropriately take advantage of moments of order $2$ in \pac-Bayes bounds, 
and a possible direction for that has been proposed by~\cite{lacasse07pacbayes}. This
is left for future work as it is not central to the present paper.

Of course, the ranking rule may be
based on a scoring function $f\in\realset^{\inputspace}$ such that
$h(X,X')=f(X)-f(X')$, in which case all the results that we state 
in terms of $h$ can be stated similarly in terms of $f$. This is 
important to note from a practical point of view as it is probably
more usual to learn functions defined over $\inputspace$ rather
than $\inputspace\times\inputspace$ (as is $h$).

Finally, we would like to stress that the bound on $\chi^*_{\rank}$ that we have
exhibited is actually rather tight. Indeed, it is straightforward to see that the clique number 
of $\Gamma_{\rank}$ is $2(\ell-1)$ (the cliques are made of variables $\{Z_{ip}\}_{p}\union\{Z_{pi}\}_{p}$ for every $i$), 
and according to Property \ref{prop:fc}, $2(\ell-1)$ is therefore a lower bound on $\chi^*_{\rank}$. If $\ell$ is even,
then our bound on $\chi^*_{\rank}$ is equal to $2(\ell-1)$ and so is $\chi^*_{\rank}$; if $\ell$ is odd, then our bound is
$2\ell$.

\subsection{Bipartite Ranking and a Bound on the \auc}
A particular ranking setting is that of bipartite ranking, where $\outputspace=\{-1,+1\}$.
Let ${D}$ be a distribution over ${\inputspace}\times{\outputspace}$ and
${D}_{+1}$ (${D}_{-1}$) be the class conditional distribution ${D}_{X|Y=+1}$
(${D}_{X|Y=-1}$) with respect to ${D}$.  In this setting
(see, e.g. \cite{agarwal05generalization}), one may be interested in controlling what we call the
{\em bipartite misranking risk} $\risk^{\auc}(h)$ (the reason for the \auc superscript  will become clear in the sequel), of a ranking rule $h\in\realset^{\inputspace\times\inputspace}$ by
\begin{equation}
 R^{\auc}(h):=\proba_{\ontop{{X}\sim {D}_{+1}}{{X}'\sim{D}_{-1}}}(h(X,X')< 0).\label{eq:exrank}
\end{equation}
Note that the relation between $R^{\auc}$ and $R^{\rank}$ (cf. Equation~\eqref{eq:rankingrisk}) can be
made clear whenever the hypotheses $h$ under consideration are such that $h(x,x')$ and $h(x',x)$ have
opposite signs. In this situation, it is straightforward to see that
\begin{equation*}
R^{\rank}(h)=2\proba(Y=+1)\proba(Y=-1)\risk^{\auc}(h).
\end{equation*}

Let $\bfS=\{({X}_i,{Y}_i)\}_{i=1}^\ell$ be an \iid sample distributed according to
 ${\bfD}_\ell={D}^\ell$. The empirical bipartite ranking risk $\hat{R}^{\auc}(h,\bfS)$
of $h$ on $\bfS$ defined as
\begin{equation}
\hat{R}^{\auc}(h,\bfS):=\frac{1}{\ell^+\ell^-}\sum_{\ontop{i:{Y}_i=+1}{ j:{Y}_j=-1}}\indicator_{h(X_i,X_j)<0}\label{eq:emprank} 
\end{equation}
where $\ell^+$ ($\ell^-$) is the number of positive (negative) data in $\bfS$,
estimates the
fraction of pairs $(X_i,X_j)$ that are incorrectly ranked
incorrectly (given that ${Y}_i=+1$ and ${Y}_j=-1$)
by $h$: it is an unbiased estimator of $R^{\auc}(h)$. 

As before, $h$ may be expressed in terms of a scoring function $f\in\realset^{\inputspace}$  such that $h(X,X')=f(X)-f(X')$, in which
case (overloading notation):
\begin{equation*}
 R^{\auc}(f)=\proba_{\ontop{{X}\sim {D}_{+1}}{{X}'\sim{D}_{-1}}}(f(X)<f(X')) \text{ and } \hat{R}^{\auc}(f,\bfS)=\frac{1}{\ell^+\ell^-}\sum_{\ontop{i:{Y}_i=+1}{ j:{Y}_j=-1}}\indicator_{f(X_i)<f(X_j)},
\end{equation*}
 where we recognize in $\hat{R}^{\auc}(f,\bfS)$ one minus the Area under the {\sc Roc} curve, or \auc,
of $f$ on $\bfS$ \citep{agarwal05generalization,cortes04auc}, hence the \auc superscript in the
name of the risk. As a consequence, providing a \pac-Bayes bound on $\risk^{\auc}(h)$ (or $\risk^{\auc}(f)$) amounts
to providing a generalization (lower) bound on the \auc, which is a widely used measure in practice to evaluate
the performance of a scoring function.

Let us define $X_{ij}:=(X_i,X_j)$, $Z_{ij}:=(X_{ij},1)$ and
$\bfZ:=\{Z_{ij}\}_{ij:Y_i=+1,Y_j=-1}$, i.e. $\bfZ$ is a sequence of pairs
$X_{ij}$ made of one positive example and one negative example. We then are once again
in the framework defined earlier\footnote{The slight difference with
  what has been described above is that the dependency graph is now a
  random variable: it depends on the $Y_i$'s. It is shown in the proof of Theorem~\ref{th:pbauc} how this can be dealt with.}, i.e., the $Z_{ij}$'s share the same distribution but are dependent
on each other, since $Z_{ij}$ depends on $\{Z_{pq}:p=i\text{ or }
q=j\}$ (see Figure~\ref{fig:examples}). Note that in order to ease the
reading of the present subsection, we make
the implicit decomposition of training set $\bfS$ into
$\bfS=\bfS^+\cup\bfS^-$, where $\bfS^+$ (resp. $\bfS^-$) is
made of the $\ell^+$ ($\ell^-$) positive (negative) data of
$\bfS$; the size $\ell$ of $\bfS$ is therefore $\ell=\ell^++\ell^-$.
 This decomposition entails a separate reindexing of the
 positive (negative) data from $1$ to $\ell^+$ (from $1$ to $\ell^-$).

Building on Theorem~\ref{th:cpbb2}, we have the following result:
\begin{theorem}[\auc \pac-Bayes bound]
\label{th:pbauc}
$\forall {D}$ over ${\inputspace}\times{\outputspace}$, $\forall{\family}\subseteq\realset^{{\inputspace}\times\inputspace}$, $\forall\delta\in(0,1]$, $\forall {P}$, with
probability at least $1-\delta$ over the random draw of $\bfS\sim {D}^\ell$, the following holds
  \begin{equationsize}{\small}
    \forall {Q},\;\kl(\hat{e}_{{Q}}^{\auc}(\bfS)||e_{{Q}}^{\auc})\leq\frac{1}{\ell_{\min}}\left[\KL({Q}||{P})+\ln\frac{\ell_{\min}+1}{\delta}\right],\label{eq:pbauc}
\end{equationsize}
where $\ell_{\min}=\min(\ell^+,\ell^-)$, and 
\begin{align*}
\hat{e}_{{Q}}^{\auc}(\bfS)&:=\expectation_{h\sim {Q}}\hat{R}^{\auc}(h,\bfS)\\
e_{{Q}}^{\auc}&:=\expectation_{\bfS\sim D^{\ell}}\hat{e}_{{Q}}^{\auc}(\bfS).
\end{align*}
\end{theorem}
\begin{proof}
The proof works in three steps and borrows ideas from~\cite{agarwal05generalization}. The first
two parts are necessary to deal with the fact that the dependency graph of $\bfZ$, as it depends on 
the random sample $\bfS$, does not have a deterministic structure.
\paragraph{Conditioning on $\bfY=\bfy$.}
Let $\bfy\in\{-1,+1\}^\ell$ be a fixed vector and let $\ell_{\bfy}^+$
and $\ell_{\bfy}^-$ be the number of positive and negative labels in $\bfy$, respectively.
We define the distribution $\bfD_{\bfy}$ as $\bfD_{\bfy}:=\otimes_{i=1}^{\ell}D_{y_i}$;
this is a distribution on $\inputspace^{\ell}$. With a slight abuse of notation, 
${\bfD}_{\bfy}$ will also be used to denote the distribution over $({\inputspace}\times{\outputspace})^\ell$ 
of samples $\bfS=\{({X}_i,y_i)\}_{i=1}^{\ell}$ such that the sequence $\{{X}_i\}_{i=1}^{\ell}$ is distributed according to ${\bfD}_{\bfy}$.
It is easy to check that  $\forall h\in{\family}$, $\expectation_{\bfS\sim{\bfD}_{\bfy}}\hat{\risk}^{\rank}(h,\bfS)=\risk^{\rank}(h)$ (cf. equations \eqref{eq:exrank} and \eqref{eq:emprank}).

Given $\bfS$, if we define, as said earlier, $X_{ij}:=(X_i,X_j)$, $Y_{ij}:=1$ and $Z_{ij}:=(X_{ij},Y_{ij})$, then $\bfZ:=\{Z_{ij}\}_{i:y_i=1,j:y_j=-1}$ is a sample of
identically distributed variables, each with distribution $D_{\pm
  1}={D}_{+1}\otimes {D}_{-1}\otimes {\bf 1}$ over
$\inputspace\times\inputspace\times\outputspace$, where
$\outputspace=\{-1,+1\}$ and where ${\bf 1}$ is the distribution
that produces $1$ with probability $1$. 

Letting $m=\ell_{\bfy}^+\ell_{\bfy}^-$ we denote by $\bfD_{\bfy,m}$ the distribution of
the training sample $\bfZ$, within which interdependencies exist, as illustrated in
Figure~\ref{fig:examples}. Theorem~\ref{th:cpbb2} can thus be directly
applied to classifiers trained on $\bfZ$, the structure of
$\Gamma(\bfD_{\bfy,m})$ and its corresponding fractional chromatic
number $\chi_{\bfy}^*$ being completely determined by $\bfy$. Hence, letting
$\family\subseteq\realset^{\inputspace\times\inputspace}$, we have: $\forall\delta\in(0,1]$, $\forall P$ over $\family$, with probability at least $1-\delta$ over the random draw of $\bfZ\sim\bfD_{\bfy,m}$,
\begin{equationsize*}{\normalsize}
\forall Q,\;\kl(\hat{e}_Q(\bfZ)||e_Q)\leq\frac{\chi_{\bfy}^*}{m}\left[\KL(Q||P)+\ln\frac{m+\chi_{\bfy}^*}{\delta\chi_\bfy^*}\right],
\end{equationsize*}
where $\hat{e}_Q(\bfZ)=\expectation_{h\sim Q}\hat{\risk}(h,\bfZ)=\expectation_{h\sim Q}\sum_{ij}\indicator_{Y_{ij}h(Z_{ij})<0}=\expectation_{h\sim Q}\sum_{ij}\indicator_{h(Z_{ij})<0}$, which is exactly equal to $\hat{e}^{\auc}_Q(\bfS)$ (cf. \eqref{eq:emprank}); likewise, $e_Q=\expectation_{\bfZ\sim\bfD_{\bfy,m}}\hat{e}_Q(\bfZ)=\expectation_{\bfS\sim\bfD_{\bfy}}\hat{e}^{\auc}_Q(\bfS)=e^{\auc}_Q$. Hence, 
$\forall\delta\in(0,1]$, $\forall P$, with probability 
at least $1-\delta$ over the random draw of $\bfS\sim\bfD_{\bfy}$,
\begin{equationsize}{\normalsize}
\forall Q,\;\kl(\hat{e}_{Q}^{\auc}(\bfS)||e_{Q}^{\auc})\leq\frac{\chi_{\bfy}^*}{m}\left[\KL(Q||P)+\ln\frac{m+\chi_{\bfy}^*}{\delta\chi_\bfy^*}\right].\label{eq:event}
\end{equationsize}

\paragraph{Unconditioning on $\bfY$.}
As proposed by~\cite{agarwal05generalization}, let us call
$\Phi(P,\bfS,\delta)$ the event~(\ref{eq:event}); we just stated
that $\forall \bfy\in\{-1,+1\}^\ell$, $\forall
P$, $\forall\delta\in(0,1]$, $\proba_{\bfS\sim{\bfD}_{\bfy}}(\Phi(P,\bfS,\delta))\geq
1-\delta$, or, equivalently
$$\proba_{\bfS\sim{\bfD}_{\ell}}(\neg\Phi(P,\bfS,\delta)|Y=\bfy)=\proba_{\bfS\sim{\bfD}_{\bfy}}(\neg\Phi(P,\bfS,\delta))<\delta,$$
i.e., the conditional (to $Y=\bfy$) probability of the event
$\neg\Phi(P,\bfS,\delta)$ is bounded by $\delta$. This directly
implies that the unconditional probability of
$\neg\Phi(P,\bfS,\delta)$ is bounded by
$\delta$ as well:
$$\proba_{\bfS\sim{\bfD}_{\ell}}(\neg\Phi(P,\bfS,\delta))\leq \proba_{\bfS\sim{\bfD}_{\ell}}(\neg\Phi(P,\bfS,\delta)|Y=\bfy)<\delta.$$

Hence,
$\forall\delta\in(0,1]$, $\forall P$, with probability 
at least $1-\delta$ over the random draw of $\bfS\sim{\bfD}_\ell$,
\begin{equationsize}{\normalsize}
\forall Q,\;\kl(\hat{e}_{Q}^{\auc}||e_{Q}^{\auc})\leq\frac{\chi_{\bfS}^*}{m_{\bfS}}\left[\KL(Q||P)+\ln\frac{m_{\bfS}+\chi_{\bfS}^*}{\delta\chi_\bfS^*}\right].\label{eq:chis}
\end{equationsize}
where $\chi^*_{\bfS}$ is the fractional chromatic number of the graph
$\Gamma(\bfZ)$, with $\bfZ$ defined from $\bfS$ as in the first part of the proof, where
the observed (random) labels are now taken into account; here $m_{\bfS}=\ell^+\ell^-$,
where $\ell^+$ ($\ell^-$) is the number of positive
(negative) data in $\bfS$.

\paragraph{Computing  the Fractional Chromatic Number.}
In order to finish the proof, it suffices to observe that, for
$\bfZ=\{Z_{ij}\}_{ij}$, if
$\ell_{\max}=\max(\ell^+,\ell^-)$, then the fractional chromatic number
of $\Gamma(\bfZ)$ is $\chi^*=\ell_{\max}$.  

Indeed, the clique
number of $\Gamma(\bfZ)$ is $\ell_{\max}$ as for all
$i=1,\ldots,\ell^+$ ($j=1,\ldots,\ell^-$),
$\{Z_{ij}:j=1,\ldots,\ell^-\}$ ($\{Z_{ij}:i=1,\ldots,\ell^+\}$) defines a
clique of order $\ell^-$ ($\ell^+$) in $\Gamma(\bfZ)$. Thus, from
Property~\ref{prop:fc}: $\chi\geq\chi^*\geq\ell_{\max}$.  

A proper exact cover
$\bfC=\{C_k\}_{k=1}^{\ell_{\max}}$ of $\Gamma(\bfZ)$ can be
constructed as follows. Suppose that
$\ell_{\max}=\ell^+$, then $C_k=\{Z_{i\sigma_k(i)}:i=1,\ldots,\ell^-\}$,
 with $$\sigma_k(i)=(i+k-2 \mod \ell^+)+1,$$
is an independent set: no two variables $Z_{ij}$ and $Z_{pq}$ in $C_k$ are such
that $i=p$ or $j=q$. In addition, it is straightforward to check that $\bfC$ is indeed a cover of $\Gamma(\bfZ)$.
This cover is of size $\ell^+=\ell_{\max}$, which means that it achieves
the minimal possible weight over proper exact (fractional) covers since $\chi^*\geq\ell_{\max}$. Hence, $\chi^*=\chi=\ell_{\max}(=c(\Gamma))$.
Plugging in this value of $\chi^*$ in~(\ref{eq:chis}), and noting that $m_{\bfS}=\ell_{\max}\ell_{\min}$ with $\ell_{\min}=\min(\ell^+,\ell^-)$,
closes the proof.
\end{proof}
We observe that in the theorem, the dependence on the skew of the sample is expressed in
terms of $1/\min(\ell^+,\ell^-)$, whereas in the the works of~\cite{agarwal05generalization} and \cite{usunier05data},
 the bound depends on the larger $1/\ell^++1/\ell^-$.

The \pac-Bayes bound of Theorem~\ref{th:pbauc} can be specialized to the
case where $h(x,x')=f(x)-f(x')$ with $f\in\{x\mapsto w\cdot x: w\in\inputspace\}$:
$f$ is therefore a linear scoring function and $h(x,x')=w\cdot(x-x')$. The ranking rule
$h$ is thus a linear classifier acting on the difference of its arguments (the next result we present therefore carries over to kernel classifiers).
As proposed by~\cite{langford05tutorial}, we may assume an isotropic Gaussian prior
$P=\normal(0,\identity)$ and a family of posteriors $Q_{w,\mu}$ parameterized by $w\in\overline{\inputspace}$ and $\mu>0$ such
that $Q_{w,\mu}$ is $\normal(\mu,1)$ in the direction $w$ and $\normal(0,1)$ in all perpendicular directions, we arrive at the 
following theorem:
\begin{theorem}[\auc Linear \pac-Bayes bound]
\label{th:pbaucgaussian}
$\forall\ell,\forall{D}$ over ${\inputspace}\times{\outputspace}$, $\forall\delta\in(0,1]$,
the following holds with probability at least $1-\delta$ over the draw of $\bfS\sim{D}^{\ell}$:
\begin{equationsize*}{\small}
\forall w,\mu>0,\;\kl(\hat{e}^{\auc}_{Q_{w,\mu}}(\bfS)||e^{\auc}_{Q_{w,\mu}})\leq \frac{1}{\ell_{\min}}\left[\frac{\mu^2}{2}+\ln\frac{\ell_{\min}+1}{\delta}\right].
\end{equationsize*}
\end{theorem}
\begin{proof}
Straightforward from the bound of \cite{langford05tutorial} and Theorem~\ref{th:pbauc}.
\end{proof}
Note that this specific parametrization of $Q$ could have been done in
Theorem~\ref{th:pbranking} as well.  We arbitrarily choose to provide
it for this \auc based bound as learning linear ranking rule by \auc
minimization is a common approach \citep{ataman06learning,brefeld05auc,rakotomamonjy04optimizing}, and the presented result may
be of practical interest (for model selection purpose, for instance) for a larger audience.

The bounds given in Theorem~\ref{th:pbauc} and
Theorem~\ref{th:pbaucgaussian} are very similar to what we would get
if applying \iid \pac-Bayes bound to one (independent) element $C_j$
of a minimal cover (i.e. its weight equals the fractional chromatic
number) $\bfC=\{C_j\}_{j=1}^n$ such as the one we used in
the proof of Theorem~\ref{th:pbauc}. This would
imply the empirical error $\hat{e}_{\overline{Q}}^{\rank}$ to be computed on only one specific $C_j$ and not all
the $C_j$'s simultaneously, as is the case for the new results. It turns out that, for
proper exact fractional covers $\bfC=\{(C_j,\omega)\}_{j=1}^n$ with elements $C_j$
having the same size, it is better, in terms of absolute moments of the empirical error, to
assess it on the whole dataset, rather than on only
one $C_j$. The following proposition formalizes this.
\begin{proposition}
\label{prop:moments}
$\forall \bfD_m$, $\forall\family$, $\forall\bfC=\{(C_j,\omega_j\}_{j=1}^n\in\pefc(\bfD_m)$,
$\forall Q$, $\forall r\in\naturalset, r\geq 1$, if $|C_1|=\ldots=|C_n|$ then
\begin{equationsize*}{\normalsize}
\expectation_{\bfZ\sim\bfD_m}|\hat{e}_Q(\bfZ)-e_Q|^r\leq\expectation_{\bfZ^{(j)}\sim\bfD_m^{(j)}}|\hat{e}_Q(\bfZ^{(j)})-e_Q|^r,\forall j\in\{1,\ldots n\}.
\end{equationsize*}
\end{proposition}
\begin{proof}
Using the convexity of $|\cdot|^r$ for $r\geq 1$, the linearity of $\expectation$
and the notation of section~\ref{sec:bounds}, for $\bfZ\sim\bfD_m$:
\begin{alignsize*}{\normalsize}
|\hat{e}_Q(\bfZ)-e_Q|^r&=|\sum_{j}\pi_j(\expectation_{h\sim Q}\hat{\risk}(h,\bfZ^{(j)})-\risk(h))|^r\\
&\leq\sum_{j}\pi_j|\expectation_{h\sim Q}(\hat{\risk}(h,\bfZ^{(j)})-\risk(h))|^r\\
&=\sum_{j}\pi_j|\hat{e}_Q(\bfZ^{(j)})-e_Q|^r.
\end{alignsize*}
Taking the expectation of both sides with respect to $\bfZ$ and noting that the random variables $|\hat{e}_Q(\bfZ^{(j)})-e_Q|^r$,
 have the same distribution, gives the result.
\end{proof}
This proposition upholds the idea of \cite{pemmaraju01equitable} to base the decomposition
of a dependency graph on equitable coloring.

\subsection{$\beta$-mixing Processes}
Here, we provide a \pac-Bayes theorem for classifiers trained 
on data from a stationary $\beta$-mixing process, of which we recall some definitions,
as formulated by \cite{yu94rates} (see also, e.g., also \cite{mohri09rademacher}).
\begin{definition}[Stationarity]
A sequence of random variables $\bfZ=\{Z_t\}_{t=-\infty}^{+\infty}$ is {\em stationary}
if, for any $t$ and nonnegative integer $m$ and $k$, the random subsequences $(Z_t,\ldots,Z_{t+m})$ and $(Z_{t+k},\ldots,Z_{t+m+k})$
are identically distributed.
\end{definition}

\begin{definition}[$\beta$-mixing process]
Let $\bfZ=\{Z_t\}_{t=-\infty}^{+\infty}$ be a stationary
sequence of random variables. For any $i,j\in\integerset\cup\{-\infty,+\infty\}$, let $\sigma_i^j$ denote
the $\sigma$-algebra generated by the random variables $Z_k$, $i\leq k\leq j$. Then, for any positive
integer $k$, the $\beta$-mixing coefficient $\beta(k)$ of the stochastic process $\bfZ$ is defined as
\begin{equation}
\label{eq:betacoefficient}
\beta(k)=\sup_{n\geq 1}\expectation\sup\left\{\left|\proba\left(A|\sigma_1^n\right)-\proba\left(A\right)\right|:A\in\sigma_{n+k}^{+\infty}\right\}.
\end{equation} 
$\bfZ$ is said to be $\beta$-mixing if $\beta(k)\rightarrow 0$ when $k\rightarrow \infty$.
\end{definition}
(Note there is an equivalent definition of the $\beta$-mixing
coefficient based on finite partitions; see \cite{yu94rates} for details.)
Stationary $\beta$-mixing processes model a situation where the interdependence between the random
variables at hand is temporal. When the process is mixing, it means that the strengh of dependence
between variables weakens over times. 

The bound that we propose is in the same vein as the one proposed by \cite{mohri09rademacher}, with
the difference that our bound is a \pac-Bayes bound and theirs a Rademacher-complexity-based bounds.
In addition to being a new type of data-dependent bound for the case of stationary $\beta$-mixing process,
we may anticipate that, in practical situations, our bound inherits the tightness of the \iid \pac-Bayes
bound (whereas, to the best of our knowledge, there is no evidence of such practicality for Rademacher-complexity-based bounds).

Let us state our generalization bound for classifiers trained on samples $\bfZ$ drawn from 
stationary $\beta$-mixing distributions.
\begin{theorem}[$\beta$-mixing process \pac-Bayes bound]
\label{th:pbbetamixing}
Let $m$ be a positive integer. Let $\bfD^{\beta}$ be a stationary $\beta$-mixing distribution over $\productspace$ and $\bfD^{\beta}_m$ be the distribution of $m$-samples according to $\bfD^{\beta}$.
$\forall{\family}\subseteq\realset^{\inputspace}$, $\forall \mu, a \in\naturalset$ such that $2\mu a=m$, $\forall\delta\in(2(\mu-1)\beta(a),1]$, $\forall {P}$, with
probability at least $1-\delta$ over the random draw of $\bfZ\sim \bfD^{\beta}_m$, the following holds
  \begin{equationsize}{\small}
    \forall {Q},\;\kl(\hat{e}^{\beta}_{Q}(\bfZ)||e^{\beta}_{Q})\leq\frac{1}{\mu}\left[\KL({Q}||{P})+\ln\frac{2(\mu+1)}{\delta-2(\mu-1)\beta(a)}\right],\label{eq:pbbetamixing}
\end{equationsize}
where
\begin{align*}
\hat{e}_{Q}^{\beta}(\bfZ)&:=\expectation_{h\sim Q}\hat{R}(h,\bfZ)=\expectation_{h\sim Q}\sum_{t=1}^m\indicator_{Y_th(X_t)<0}\\
e_{Q}^{\beta}&:=\expectation_{\bfZ\sim \bfD^{\beta}_m}\hat{e}_{Q}^{\beta}(\bfZ).
\end{align*}
\end{theorem}
\begin{proof}
The proof makes use of the independent block decomposition proposed by \cite{yu94rates}, our chromatic \pac-Bayes bound of Theorem~\ref{th:cpbb2}, and Corollary~\ref{cor:connect} (Appendix).

\paragraph{The chromatic bound for independent blocks.} Let $\bfZ=\{Z_1,\ldots,Z_m\}$ be the random variables we have to deal with.
If $\mu$ and $a$ are two integers such that $2\mu a=m$ (we assume that $m$ is even, if it is odd one may drop the last variable $Z_m$ and work on a sample of size $m-1$).
Then $\bfZ$ can be decomposed into two subsequences $\bfZ_0$ and $\bfZ_1$ as follows:
\begin{align*}
\bfZ_0&:=\{\bfZ_0^{s}:=(Z_{a(2s-2)+1},\ldots,Z_{a(2s-2)+a}):s\in[\mu]\},\\
 \bfZ_1&:=\{\bfZ_1^s:=(Z_{a(2s-1)+1},\ldots, Z_{a(2s-1)+a}):s\in[\mu]\}.
\end{align*}  
Both $\bfZ_0$ and $\bfZ_1$ are made of $\mu$ blocks of $a$ consecutive random variables. The blocks are interdependent as well
as the variables within each block. $\bfD_0$ will denote the distribution of $\bfZ_0$.

We now define a sequence $\underline{\bfZ}$ of independent blocks as:
\begin{equation*}
\underline{\bfZ}:=\{\underline{\bfZ}^s:=(Z_1^s,\ldots,Z_a^s):s\in[\mu]\},
\end{equation*} 
such that the blocks $\underline{\bfZ}^s$ are mutually independent and such that each block
$\underline{\bfZ}^s$ has the same distribution as $\bfZ_0^s$,  that is, from the stationarity assumption, 
the distribution of $\bfZ_0^1$ (the blocks $\underline{\bfZ}^s$  are \iid). 

The dependency graph $\underline{\Gamma}$ of $\underline{\bfZ}$
is such that all the variables in a block are all connected and such that there are no connections between 
blocks. 
  Theorem~\ref{th:cpbb2} can readily be applied to the random sample $\underline{\bfZ}$, whose distribution we denote $\underline{\bfD}$: for all $P$ and $\delta\in(0,1]$,
\begin{equation}
\label{eq:klevent}
\proba_{\underline{\bfZ}\sim\underline{\bfD}}\left(\Phi(P,\underline{\bfZ},\delta)\right)<\delta,
\end{equation}
with $e_Q:=\expectation_{\underline{\bfZ}\sim\underline{\bfD}}\hat{e}_Q(\underline{\bfZ})$ and
 $\Phi(P,\bfZ,\delta)$ is the event defined as:
$$\Phi(P,\bfZ,\delta):=\left\{\exists Q,\; \kl(\hat{e}_Q(\bfZ)||e_{Q})> \frac{1}{\mu}\left[\KL(Q||P)+\ln \frac{\mu  + 1}{\delta}\right]\right\}.$$

 To see why and how Theorem~\ref{th:cpbb2} can be used to get \eqref{eq:klevent}, observe that:
\begin{itemize}
\item the number of variables in $\underline{\bfZ}$ is $\mu a$;
\item by stationarity, all variables $Z_\alpha^s$, for $\alpha\in[a]$ and $s\in[\mu]$ share the same distribution: we therefore do actually work with dependent but identically distributed variables;
\item the (fractional) chromatic number $\underline{\chi}^*$ of $\underline{\Gamma}$ is $a$, since 
\begin{enumerate}
  \item the clique number is $a$ (i.e. the number
of variables in each block),
\item  the cover $\underline{\bfC}$ of $\underline{\Gamma}$ with $$\underline{\bfC}:=\left\{\left(C_{\alpha}:=\{Z_\alpha^1,\ldots,Z_\alpha^\mu\},1\right)\right\}_{1\leq \alpha \leq a}$$
is a proper exact cover of size $a$.
\end{enumerate}
\end{itemize}
Noting that, consequently
\begin{equation*}
\frac{\underline{\chi}^*}{\mu a}= \frac{a}{\mu a}=\frac{1}{\mu}\qquad\text{ and }\qquad
\frac{\mu a + \underline{\chi}^*}{\delta \underline{\chi}^*}=\frac{\mu a +a}{\delta a} = \frac{\mu + 1}{\delta}
\end{equation*}
 gives the expression of $\Phi(P,\bfZ,\delta)$ and~\eqref{eq:klevent}.

The last two steps of the proof are similar to those
used by \cite{mohri09rademacher} to establish their bound. 

\paragraph{A bound for $\bfZ_0$.} To establish the bound for $\bfZ_0$, it suffices
to use Corollary~\ref{cor:connect} (Appendix) with $c(\bfz)$ being defined as:
$$c(\bfz):=\indicator_{\Phi(P,\bfz,\delta)},$$
which is a bounded measurable function on the blocks $\bfZ_0^s$ (and thus on the blocks $\underline{\bfZ}_s$).
We have:
$$\left|\expectation_{\bfZ_0\sim\bfD_0}c(\bfZ_0)-\expectation_{\underline{\bfZ}\sim\underline{\bfD}}c(\underline{\bfZ})\right|\leq (\mu-1)\beta(a),$$
and therefore, since $\proba_{\bfZ_0\sim\bfD_0}(\Phi(P,\bfZ_0,\delta))=\expectation_{\bfZ_0\sim\bfD_0}c(\bfZ_0)$ and $\proba_{\underline{\bfZ}\sim\underline{\bfD}}\left(\Phi(P,\underline{\bfZ},\delta)\right)=\expectation_{\underline{\bfZ}\sim\underline{\bfD}}c(\underline{\bfZ})$:
\begin{align}
\label{eq:pbz0}
\proba_{\bfZ_0\sim\bfD_0}(\Phi(P,\bfZ_0,\delta))&\leq \proba_{\underline{\bfZ}\sim\underline{\bfD}}\left(\Phi(P,\underline{\bfZ},\delta)\right) + (\mu-1)\beta(a) \\
&<\delta +  (\mu-1)\beta(a) \tag{cf. \eqref{eq:klevent}}.
\end{align}

\paragraph{Establishing the bound.}  Finally, observe that:
\begin{align*}
\Phi(P,\bfZ,\delta)
&\Rightarrow \exists Q:\; \frac{1}{2}\kl(\hat{e}_Q(\bfZ_0)||e_{Q})+ \frac{1}{2}\kl(\hat{e}_Q(\bfZ_1)||e_{Q})> \frac{1}{\mu}\left[\KL(Q||P)+\ln \frac{\mu  + 1}{\delta}\right]\\
&\Rightarrow \exists Q:\;\bigvee_{i\in\{0,1\}} \left\{\kl(\hat{e}_Q(\bfZ_i)||e_{Q})) > \frac{1}{\mu}\left[\KL(Q||P)+\ln \frac{\mu  + 1}{\delta}\right]\right\}\\
&\Rightarrow \bigvee_{i\in\{0,1\}} \left\{\exists Q:\kl(\hat{e}_Q(\bfZ_i)||e_{Q})) > \frac{1}{\mu}\left[\KL(Q||P)+\ln \frac{\mu  + 1}{\delta}\right]\right\}\\
&\Leftrightarrow \Phi(P,\bfZ_0,\delta)\vee \Phi(P,\bfZ_1,\delta),
\end{align*}
where we used $\hat{e}_Q(\bfZ)=\hat{e}_Q(\bfZ_0)/2 +\hat{e}_Q(\bfZ_1)/2$ and the convexity of $\kl$ in the first line.

This leads to:
\begin{align*}
  \proba_{\bfZ\sim\bfD_m^{\beta}}(\Phi(P,\bfZ,\delta))&\leq\proba_{\bfZ\sim\bfD_m^{\beta}}(\Phi(P,\bfZ_0,\delta)\vee \Phi(P,\bfZ_1,\delta))\\
  &\leq \proba_{\bfZ\sim\bfD_m^{\beta}}(\Phi(P,\bfZ_0,\delta))+\proba_{\bfZ\sim\bfD_m^{\beta}}(\Phi(P,\bfZ_1,\delta))\tag{union bound}\\
  &=2\proba_{\bfZ\sim\bfD_m^{\beta}}(\Phi(P,\bfZ_0,\delta))\tag{stationarity}\\
  &=2\proba_{\bfZ_0\sim\bfD_0}(\Phi(P,\bfZ_0,\delta))\tag{marginalization wrt $\bfZ_0$}\\
  &\leq 2\delta + 2(\mu-1)\beta(a). \tag{cf. \eqref{eq:pbz0}}
\end{align*}
Adjusting $\delta$ to $\delta/2-(\mu-1)\beta(a)$ ends the proof.
\end{proof}


\section{Conclusion}
\label{sec:conclusion}
In this work, we propose the first \pac-Bayes bounds applying for
classifiers trained on non-\iid data.  The derivation of these results
rely on the use of fractional covers of graphs, convexity and standard
tools from probability theory. The results that we provide are very
general and can easily be instantiated for specific learning settings
such as ranking and learning from from mixing distributions: amazingly,
we obtain at a very low cost original \pac-Bayes bounds for these settings.
Using a generalized \pac-Bayes bound, we provide in the appendix a
chromatic \pac-Bayes bound that holds for non-independently and non-identically
distributed data: it allows us to derive a \pac-Bayes bound for classifiers
trained on data from a stationary $\varphi$-mixing distribution.

This work gives rise to many interesting questions. First, it seems
that using a fractional cover to decompose the non-\iid training data
into sets of \iid data and then tightening the bound through the use
of the chromatic number is some form of variational relaxation as
often encountered in the context of inference in graphical models, the
graphical model under consideration in this work being one that
encodes the dependencies in $\bfD_m$. It might be interesting to make
this connection clearer to see if, for instance, tighter and still
general bounds can be obtained with more appropriate variational
relaxations than the one incurred by the use of fractional covers.

Besides, Theorem~\ref{th:cpbb1} advocates for the learning algorithm
described in Remark~\ref{rem:algo}. We would like to see how
such a learning algorithm based on possibly multiple priors/multiple
posteriors could perform empirically and how tight the proposed bound
could be.

On another empirical side, it might be interesting to run
simulations on bipartite ranking problems to see how accurate the
bound of Theorem~\ref{th:pbaucgaussian} can be: we expect the results
to be of good quality, because of the resemblance of the bound of the
theorem with the \iid \pac-Bayes theorem for margin
classifiers, which has proven to be rather accurate
\cite{langford05tutorial}. The work of \cite{germain09pacbayesian} is also
another contribution that tends to support that a practical use of 
our bounds should provide competitive results (note that Theorem~\ref{cor:generic}
gives a sufficient condition for the general \pac-Bayes bound of \cite{germain09pacbayesian} to 
be non degenerate). Likewise, it would be interesting
to see how the possibly more accurate \pac-Bayes bound for large
margin classifiers proposed by~\cite{langford02pac}, which should
translate to the case of bipartite ranking as well, performs
empirically.
The question also remains as to what kind of strategies to learn
the prior(s) could be used to render the bound of
Theorem~\ref{th:cpbb1} the tightest possible. This is one of the most
stimulating question as performing such prior learning makes it
possible to obtain very accurate generalization bound \cite{ambroladze07tighter}.

The connection between our ranking bounds and the theory of U-statistics
makes it possible to envision the use of higher order moments in establishing
\pac-Bayes bounds, thanks to Hoeffding's decomposition. We plan to 
investigate further in this direction, for both the ranking measures we
have studied (noting that the \auc is a two-sample U-statistics \citep{hoeffding63probability}).

Finally, we have been working on a more general way to establish
chromatic bounds from \iid bounds (covering VC, Rademacher, \pac-Bayes
and --~possibly~-- binomial tail bounds), without the need to perform
`low-level' calculations such as the ones proposed in
section~\ref{sec:proof}.  The meta-bound that we have been developping is in the
spirit of that proposed by \cite{blanchard07occam}, except that the
randomization we propose is on the subsets constituting the fractional
cover (and not the hypothesis set). In other terms, given a cover $\bfC=\{(C_j,\omega_j)\}_j$, the fact that an
\iid bound holds on one subset $C_j$ of a cover is considered as a
random event, the probability of a subset to be chosen being $\omega_j/\omega(\bfC)$.
A simple union bound gives our generic result, which translates into  cover-independent
(but fractional-chromatic-number-dependent) chromatic bounds such as \eqref{eq:cpbb2}
(Theorem~\ref{th:cpbb2}) under very mild conditions on the shape of
the base \iid bound. Along with that work, we try to answer the
question of establishing a principled way to handle situations where random variables 
show weak dependencies (as is the case for $\beta$-mixing processes),
as for now, the framework described here  
applies when variables are either dependent or independent,
disregarding the magnitude of the dependencies -- our \pac-Bayes bound for
  $\beta$-mixing processes would then be a specific case of such general
  result.

 
\section*{Acknowledgment}
This work is partially supported by the IST
Program of the EC, under the FP7 Pascal 2 Network of Excellence, ICT-216886-NOE.
\section{Appendix}

\subsection{Technical Lemmas}
\begin{lemma}
\label{lem:entropy}
Let $D$ be a distribution over $\productspace$.
$$\forall h\in\family, \expectation_{\bfZ\sim D^m}e^{m\kl\left(\hat{\risk}(h,\bfZ)||\risk(h)\right)}\leq m+1.$$
\end{lemma}
\begin{proof}
Let $h\in\family$. For $\bfz\in \productspace^m$, we let $q(\bfz)=\hat{\risk}(h,\bfz)$; we also
let $p=R(h)$. Note that since $\bfZ$ is i.i.d, $mq(\bfZ)$ is binomial with parameters $m$ and $p$ (recall
that $r(h,Z)$  takes the values $0$ and $1$ upon correct and erroneous classification of $Z$ by $h$, respectively).
\begin{alignsize*}{\normalsize}
\expectation_{\bfZ\sim{D^m}}e^{m\kl\left(q(\bfZ)||p\right)}&=\sum_{\bfz\in\productspace^m}e^{m\kl\left(q(\bfz)||p\right)}\proba_{\bfZ\sim D^m}(\bfZ=\bfz)\\
&=\sum_{0\leq k\leq m}e^{m\kl\left(\frac{k}{m}||p\right)}\proba_{\bfZ\sim D^m}(mq(\bfZ)=k)\\
&=\sum_{0\leq k\leq m}\binom{m}{k}e^{m\kl\left(\frac{k}{m}||p\right)}p^{k}(1-p)^{m-k}\\
&=\sum_{0\leq k\leq m}\binom{m}{k}e^{m\left(\frac{k}{m}\ln\frac{k}{m}+(1-\frac{k}{m})\ln(1-\frac{k}{m})\right)}\\
&=\sum_{0\leq k\leq m}\binom{m}{k}\left(\frac{k}{m}\right)^{k}\left(1-\frac{k}{m}\right)^{m-k}.
\end{alignsize*}
However, it is obvious that, from the definition of the binomial distribution,
$$ \forall m\in\naturalset,\forall k\in[0,m],\forall t\in[0,1], \binom{m}{k}t^{k}(1-t)^{m-k}\leq 1.$$
This is obviously the case for $t=\frac{k}{m}$, which gives
\begin{alignsize*}{\normalsize}
\sum_{0\leq k\leq m}\binom{m}{k}\left(\frac{k}{m}\right)^{k}\left(1-\frac{k}{m}\right)^{m-k}\leq \sum_{0\leq k\leq m}1=m+1.
\end{alignsize*}
\end{proof}

\begin{theorem}[Jensen's inequality]
\label{th:jensen}
Let $f\in\realset^\inputspace$ be a convex function. For all probability distribution
$P$ on $\inputspace$:
$$f(\expectation_{X\sim P}X)\leq\expectation_{X\sim P}f(X).$$
\end{theorem}
\begin{theorem}[Markov's Inequality]
\label{th:markov}
Let $X$ be a positive random variable on $\realset$, such that $\expectation X<\infty$.
\begin{equationsize*}{\normalsize}
\forall t\in\realset,\proba_X\left\{X\geq\frac{\expectation X}{t}\right\}\leq\frac{1}{t}.
\end{equationsize*}
Consequently: $\forall M\geq\expectation X,\forall t\in\realset,\proba_X\left\{X\geq\frac{M}{t}\right\}\leq\frac{1}{t}.$
\end{theorem}
\begin{lemma}[Convexity of kl]
\label{lem:kl}
$\forall p,q,r,s\in[0,1],\forall \alpha\in[0,1],$ 
\begin{equation*}
\kl(\alpha p + (1-\alpha)q||\alpha r + (1-\alpha)s)\leq \alpha\kl(p||r)+(1-\alpha)\kl(q||s).
\end{equation*}
\end{lemma}
\begin{proof}
It suffices to see that $f\in\realset^{[0,1]^2}, f({\bf v}=[p\; q])=\kl(q||p)$ is convex
over $[0,1]^2$: the Hessian $H$ of $f$ 
is
\begin{equation*}
H=\left[\begin{array}{cc}
     \frac{q}{p^2}+\frac{1-q}{(1-p)^2}   & -\frac{1}{p}-\frac{1}{1-p}\\
     -\frac{1}{p}-\frac{1}{1-p}  & \frac{1}{q}+\frac{1}{1-q} 
\end{array}\right],
\end{equation*}
and, for $p,q\in[0,1]$, $\frac{q}{p^2}+\frac{1-q}{(1-p)^2}\geq 0$ and $\det H=\frac{(p-q)^2}{q(1-q)p^2(1-p)^2}\geq 0$: $H\succeq 0$ and $f$ is indeed convex.
\end{proof}

Finally, we have the following version by \cite{mohri09rademacher} of Corollary~2.7 in \citep{yu94rates}, which
is based on the definition of the blocks $\bfZ_k^s$:
\begin{corollary}
\label{cor:connect}
Let $c$ be a measurable function defined with respect to the blocks $\bfZ_0^s$.
If $c$ has absolute value bounded by $M$, then
$$|\expectation_{\bfZ_0\sim\bfD_0}c(\bfZ)-\expectation_{\underline{\bfZ}\sim\underline{\bfD}}c(\underline{\bfZ})|\leq (\mu-1)M\beta(a).$$
\end{corollary}

\subsection{Applications of a Generic \pac-Bayes Theorem}
Let us first recall the 
following generic \pac-Bayes result, which is a corollary/compound of results
proposed by~\cite{seeger02proof} and \cite{mcallester03simplified}. In particular, the
$\gamma$ function need not be differentiable with respect to its second argument and
it applies to any `risk' functional  $\psi$ for which a concentration inequality exists.
\begin{corollary}[Generic \pac-Bayes Theorem]
  \label{cor:generic}
 Let $\family\subseteq\realset^{\inputspace}$ and $\psi:\family\times\union_{m=1}^{\infty}\productspace^m\rightarrow \realset$.
    If there exist $\alpha\geq 1, \beta>1$ and a nonnegative convex function $\Delta:\realset\times\realset\rightarrow\realset_+$ 
    that is strictly increasing with respect to its second argument such that
  \begin{equationsize}{\normalsize}
   \forall h\in\family,\forall \varepsilon>0,\;
    \proba_{\bfZ\sim\bfD_m}\left[\expectation\psi(h)-\psi(h,\bfZ)\geq \varepsilon\right]\leq \alpha\exp\left(-\beta \Delta( \expectation\psi(h),\varepsilon)\right),
    \label{eq:genericconcentration}
  \end{equationsize}
  where $\expectation\psi(h)$ stands for $\expectation_{\bfZ\sim\bfD_m}\psi(h,\bfZ)$,
  then, $\forall P$, with probability at least $1-\delta$ over the
  draw of $\bfZ\sim\bfD_m$:
  \begin{equation}
    \label{eq:pacgeneric}
    \forall Q,\; \Delta(e_Q^{\psi},e_Q^{\psi}-\hat{e}_Q^{\psi}(\bfZ))\leq \frac{1}{\beta-1}\left[\KL(Q||P)+\ln\frac{\alpha\beta}{\delta}\right].
  \end{equation}
  where 
  \begin{align*}
    \hat{e}_Q^{\psi}(\bfZ)&:=\expectation_{h\sim Q}\psi(h,\bfZ)\\
    e_Q^{\psi}&:=\expectation_{\bfZ}\hat{e}_Q^{\psi}(\bfZ)=\expectation_{h\sim Q}\expectation_{\bfZ}\psi(h,\bfZ)
  \end{align*}
\end{corollary}
\begin{proof} Along lines from \citep{seeger02proof} and \citep{mcallester03simplified}.
\begin{enumerate}
\item Observe that, thanks to Lemma~\ref{lem:concentration} (below) with $\delta(\varepsilon):=\Delta(\expectation\psi(h),\varepsilon)$, 
$$\expectation_{\bfZ}e^{(\beta-1)\Delta(\expectation\psi(h),\expectation\psi(h)-\psi(h,\bfZ))}\leq\alpha\beta, \text{ and, } \expectation_{h\sim P}\expectation_{\bfZ}e^{(\beta-1)\Delta(\expectation\psi(h),\expectation\psi(h)-\psi(h,\bfZ))}\leq\alpha\beta$$
Applying Markov's inequality then gives:
$$\proba_{\bfZ}\left[\expectation_{h\sim P}e^{(\beta-1)\Delta(\expectation\psi(h),\expectation\psi(h)-\psi(h,\bfZ))}\geq\frac{\alpha\beta}{\delta}\right]\leq\delta$$
\item Using the entropy extremal inequality $\ln \expectation_{X\sim P}f(X)\geq -\KL(Q||P))+ \expectation_{X\sim Q}\ln f(X),$ $\forall P,Q,X$ (see the proof of Lemma~\ref{lem:l2}), and the fact that $x\mapsto \ln x$ is nondecreasing, the previous step leads to
$$\proba_{\bfZ}\left[\exists Q:-\KL(Q||P) + (\beta-1)\expectation_{h\sim Q}\Delta(\expectation\psi(h),\expectation\psi(h)-\psi(h,\bfZ))\geq\ln\frac{\alpha\beta}{\delta}\right]\leq\delta.$$
\item Since $\Delta$ is convex, Jensen's inequality can be used to give (here, $h\sim Q$)
$$\proba_{\bfZ}\left[\exists Q:-\KL(Q||P) + (\beta-1)\Delta(\expectation_{h,\bfZ}\psi(h,\bfZ),\expectation_{h,\bfZ}\psi(h,\bfZ)-\expectation_{h}\psi(h,\bfZ))\geq\ln\frac{\alpha\beta}{\delta}\right]\leq\delta.$$
\end{enumerate}
\end{proof}

\begin{lemma}[\cite{mcallester03simplified}]
\label{lem:concentration}
Let $X$ be a real-valued random variable on $\inputspace$ and $\alpha\geq 1,\beta>1$.
Let $\delta:\realset\rightarrow\realset$ be a nonnegative and strictly increasing function. We have:
$$\forall x\in\realset,\,\proba[X \geq x] \leq \alpha e^{-\beta \delta(x)}\Rightarrow \expectation\left[e^{(\beta-1)\delta(X)}\right]\leq \alpha\beta.$$
\end{lemma}
\begin{proof} See the proof of \cite{mcallester03simplified}. Here, we take $\alpha$ into account. As $f$ is strictly increasing:
\begin{align*}
\proba\left[X\geq x\right]=\proba\left[\delta(X)\geq \delta(x)\right]=\proba\left[e^{(\beta-1)\delta(X)}\geq e^{(\beta-1)\delta(x)}\right].
\end{align*}
Hence: 
$\proba\left[e^{(\beta-1)\delta(X)}\geq e^{(\beta-1)\delta(x)}\right]\leq \alpha e^{-\beta \delta(x)}.$
Setting $\nu=e^{(\beta-1)\delta(x)}$, we get: 
$$\proba\left[e^{(\beta-1)\delta(X)}\geq \nu\right]\leq \min(1,\alpha\nu^{-\beta/(\beta-1))}).$$
Thus, as for a nonnegative random variable $W$, $\expectation[W]=\int_{0}^{\infty}\proba[W\geq \nu]d\nu$:
$$\expectation\left[e^{(\beta-1)\delta(X)}\right]\leq 1+\alpha\int_{1}^{\infty}\nu^{-\beta/(\beta-1)}=1+\alpha(\beta-1).$$
Since $\alpha>1$, $1+\alpha(\beta-1)\leq \alpha\beta$, which ends the proof. 
\end{proof}

We observe that:
\begin{itemize}
\item if $\psi(h,\bfZ)=\sum_{i=1}^m\indicator_{Y_ih(Xi)<0}$ then, by the one-sided Chernoff bound, $\alpha=1$, $\beta=m$ and $\Delta(p,\varepsilon)=\kl(p-\varepsilon||p)$ make equation \eqref{eq:genericconcentration} hold. The \pac-Bayes bound
provided by Corollary~\ref{cor:generic} is that of Theorem~\ref{th:pbiid} where $m$ is replaced by $m-1$;
\item if $$\forall i\in[m],\sup_{z_1,\ldots,z_m,z_i'\in\productspace}|\psi(z_1,\ldots,z_m)-\psi(z_1,\ldots,z_{i-1},z_i',z_{i+1},\ldots,z_m)|\leq c_i,$$ then,
thanks to McDiarmid inequality \citep{mcdiarmid89method}, $\alpha=1$, $\beta=2/\sum_{i}c_i^2$ and $\Delta(p,\varepsilon)=\varepsilon^2$, make equation \eqref{eq:genericconcentration} hold and a \pac-Bayes bound can be derived (we let the reader write the corresponding \pac-Bayes bound);
\item it suffices to have an appropriate concentration inequality for the problem at hand to have an effective \pac-Bayes bound.
\end{itemize}

\subsubsection{Generalized Chromatic \pac-Bayes Bound}

To get a chromatic \pac-Bayes theorem for non-identically non-independently distributed data, we simply make use of the following
concentration inequality of~\cite{janson04large}.
\begin{theorem}[\cite{janson04large}]
\label{th:janson}
Suppose that $\bfZ=\{Z_i\}_{i=1}^m$ is an $m$-sample of real-valued
random variables distributed according to some distribution
$\bfD_m$. Suppose that each $Z_i$ has range $[a_i,b_i]$. If $S_\bfZ=\sum_{i=1}^mZ_i$, then, 
\begin{equationsize*}{\normalsize}
\forall\varepsilon>0,\;\proba_{S_\bfZ}\left[\expectation S_\bfZ - S_\bfZ\geq \varepsilon\right]\leq \exp\left[-\frac{2\varepsilon^2}{\chi^*(\bfD_m)\sum_{i=1}^{m}(b_i-a_i)^2}\right],
\end{equationsize*}
where $\chi^*(\bfD_m)$ is the fractional chromatic number of the dependency graph of $\bfD_m$.
\end{theorem}
Note that {\em no assumption} is made on the $Z_i$'s being identically distributed.

This concentration inequality gives rise to the following generalized chromatic \pac-Bayes bound
that applies to non indepently, possibly non identically distributed data and allows us to use
any bounded loss functions $r$.
\begin{theorem}[Generalized Chromatic \pac-Bayes Bound]
  \label{th:cpbb}
  $\forall \bfD_m$, $\forall\family$,
  $\forall\delta\in(0,1]$, $\forall P$, with probability at least
  $1-\delta$ over the random draw of $\bfZ\sim\bfD_m$, the following
  holds
  \begin{equationsize}{\normalsize}
    \forall
    Q,\;|\hat{e}_Q(\bfZ)-e_Q|^2\leq\frac{\chi^* M^2}{2m-\chi^* M^2}\left[\KL(Q||P)+\ln\frac{2m}{\chi^* M^2}+\ln\frac{1}{\delta}\right],\label{eq:cpbb}\end{equationsize}
  where $\chi^*$  stands for $\chi^*(\bfD_m)$, $r$ is a bounded function with range $M$
  and 
\begin{align*}
\hat{e}_Q(\bfZ)&:=\expectation_{h\sim Q}\hat{\risk}(h,\bfZ)\\
e_{Q}&:=\expectation_{h\sim Q}\hat{e}_Q(\bfZ)=\expectation_{h\sim Q}\expectation_{\bfZ\sim\bfD_m}\hat{\risk}(h,\bfZ),
\end{align*}
with $\hat{\risk}(h,\bfZ):=\sum_{i}r(h,Z_i)/m$.
\end{theorem}
\begin{proof}
It suffices to apply Corollary~\ref{cor:generic} with Theorem~\ref{th:janson}, $\alpha=1$, $\Delta(p,\varepsilon)=\varepsilon^2$ and $\beta=2m/\chi^*M$  (since, as $r$ has range $M$, $\hat{R}$ has range $M/m$).
\end{proof}
We notice the following.
\begin{itemize}
\item Here, as no assumption is done regarding the identical distribution of the $Z_i$'s, the expected risk $R(h)=\expectation_\bfZ\hat{\risk}(h,\bfZ)$ does not unfold as in \eqref{eq:iidrisk}.
\item In the case of using identically distributed random variables and the 0-1 loss, there is no concentration inequality
that allows us to retrieve the tighter \pac-Bayes bound given in
Theorem~\ref{th:cpbb2}.
\item From a more general point of view, it is enticing to try to
  establish even more generic results resting on the principle of
  graph coloring with the aim of decoupling this approach to the
  PAC-Bayesian framework. This is the subject of ongoing work.
\end{itemize}

\subsubsection{ $\varphi$-mixing \pac-Bayes Bound}
The definition of a $\varphi$-mixing process follows.
\begin{definition}[$\varphi$-mixing process] Let $\bfZ=\{Z_t\}_{t=-\infty}^{+\infty}$ be a stationary
sequence of random variables. For any $i,j\in\mathbb{Z}\cup\{-\infty,+\infty\}$, let $\sigma_i^j$ denote
the $\sigma$-algebra generated by the random variables $Z_k$, $i\leq k\leq j$. Then, for any positive
integer $k$, the $\varphi$-mixing coefficient $\varphi(k)$ of the stochastic process $\bfZ$ is defined as
\begin{equation}
\label{eq:phicoefficient}
\varphi(k)=\sup_{n,A\in\sigma_{n+k}^{+\infty},B\in\sigma_{-\infty}^n}\left|\proba\left[A|B\right]-\proba\left[A\right]\right|.
\end{equation} 
$\bfZ$ is said to be $\varphi$-mixing if $\varphi(k)\rightarrow 0$ as $k\rightarrow 0$.
\end{definition}

In order to establish our new \pac-Bayes bounds for stationary $\varphi$-mixing distributions, it suffices to
make use of the following concentration inequality by \cite{kontorovich08concentration}. 
\begin{theorem}[\cite{kontorovich08concentration}]
Let $\psi:\mathcal{U}^m\rightarrow\realset$ be a function defined over a countable space $\mathcal{U}$.
If $\psi$ is $l$-Lipschitz with respect to the Hamming metric for some $l>0$, then the following holds
for all $t>0$:
\begin{equationsize*}{\normalsize}
\proba_{\bfZ}\left[\left|\psi(\bfZ)-\expectation_{\bfZ}[\psi(\bfZ)]\right|>t\right]\leq 2\exp\left[-\frac{t^2}{2ml^2\|\Lambda_m\|_\infty^2}\right],\end{equationsize*}
where $\|\Lambda_m\|_\infty\leq 1+2\sum_{k=1}^m\varphi(k)$.
\end{theorem}

Suppose that the loss function $r$  is again such that it takes values in
$[0,M]$. Then, for any $h\in\family$, the function $\psi(\bfZ)=\frac{1}{m}\sum_{i=1}^mr(h,Z_i)=\hat{\risk}(h,\bfZ)$
is obviously $M/m$-Lipschitz. Therefore, for a sample $\bfZ$ drawn according to a
$\varphi$-mixing process, we have the following concentration inequality on $\hat{\risk}(h,\bfZ)$ that holds for any $h\in\family$:
\begin{equationsize}{\normalsize}
\proba_{\bfZ\sim\bfD_m}\left[\left|\hat\risk(h,\bfZ)-\risk(h)\right|>t\right]\leq 2\exp\left[-\frac{mt^2}{2M^2\|\Lambda_m\|_\infty^2}\right].
\label{eq:riskphi}
\end{equationsize}

We directly get the following \pac-Bayes bound for $\varphi$-mixing processes.
\begin{theorem}[\pac-Bayes bound for stationary $\varphi$-mixing processes]
\label{th:vmpb1}
Let $\bfD^{\varphi}$ be a stationary $\varphi$-mixing distribution over $\productspace$ and $\bfD^{\varphi}_m$ be the distribution of $m$-samples according to $\bfD^{\varphi}$.
$\forall{\family}\subseteq\realset^{\inputspace}$, $\forall\delta\in(0,1]$, $\forall {P}$, with
probability at least $1-\delta$ over the random draw of $\bfZ\sim \bfD^{\varphi}_m$, the following holds
\begin{alignsize*}{\normalsize}
  \forall    Q,\;|\hat{e}_{Q}^{\varphi}(\bfZ)-e_{Q}^{\varphi}|^2\leq\frac{2M^2\|\Lambda_m\|_\infty^2}{m-2M^2\|\Lambda_m\|_\infty^2}\left[\KL(Q||P)+\ln\frac{m}{M^2\|\Lambda_m\|_\infty^2}+\ln\frac{1}{\delta}\right],\label{eq:cpbb}
\end{alignsize*}
  where $\|\Lambda_m\|_\infty\leq 1+2\sum_{k=1}^m\varphi(k)$, $r(h,Z)=\indicator_{Yh(X)<0}$ 
  and 
\begin{align*}
\hat{e}_{Q}^{\varphi}(\bfZ)&:=\expectation_{h\sim Q}\hat{R}(h,\bfZ)=\expectation_{h\sim Q}\sum_{t=1}^m\indicator_{Y_th(X_t)<0}\\
e_{Q}^{\varphi}&:=\expectation_{\bfZ\sim \bfD^{\varphi}_m}\hat{e}_{Q}^{\varphi}(\bfZ).
\end{align*}
\end{theorem}
\begin{proof}
Equation~\eqref{eq:riskphi}, and Corollary \ref{cor:generic} with $\alpha=2$, $\beta=m/(2M^2\|\Lambda\|_{\infty}^2)$, $\Delta(p,\varepsilon)=\varepsilon^2$.
\end{proof}

\bibliography{pacbayesnoniid}
\end{document}